%% file: arxiv.tex
\documentclass{article}

\usepackage[preprint,nonatbib]{neurips_2022}

\usepackage[numbers]{natbib}

\usepackage{nicefrac}       
\usepackage{scalefnt,letltxmacro}
\LetLtxMacro{\oldtextsc}{\textsc}
\renewcommand{\textsc}[1]{\oldtextsc{\scalefont{1.10}#1}}
\usepackage[acronym,smallcaps,nowarn]{glossaries}
\glsdisablehyper
\makeglossaries
\usepackage{xspace}
\usepackage{float}
\usepackage{physics}
\usepackage[normalem]{ulem}

\usepackage{enumitem}

\usepackage{amssymb}
\usepackage{mathtools}
\usepackage{amsfonts}
\usepackage{amsmath}
\usepackage{amsthm,thmtools}
\usepackage{booktabs}
\usepackage[]{microtype}

\usepackage[linesnumbered,ruled,vlined]{algorithm2e}

\SetCommentSty{mycommfont}
\SetKwInput{KwInput}{Input}
\SetKwInput{KwOutput}{Output}

\usepackage[usenames,dvipsnames,svgnames,table]{xcolor}
\definecolor{pink3}{cmyk}{0, 0.7808, 0.4429, 0.1412}
\usepackage[colorlinks,linktoc=all]{hyperref}
\usepackage[all]{hypcap}
\hypersetup{citecolor=pink3}
\hypersetup{linkcolor=pink3}
\hypersetup{urlcolor =pink3}

\usepackage[capitalize,nameinlink,]{cleveref} \crefname{section}{\S}{\S\S}
\Crefname{section}{\S}{\S\S}

\definecolor{shadecolor}{gray}{0.9}

\definecolor{mylightgray}{gray}{0.94}

\makeatletter
\@ifundefined{c@rownum}{\let\c@rownum\rownum
}{}
\@ifundefined{therownum}{\def\therownum{\@arabic\rownum}}{}
\makeatother

\makeatletter
\newcommand*{\addFileDependency}[1]{\typeout{(#1)}
	\@addtofilelist{#1}
	\IfFileExists{#1}{}{\typeout{No file #1.}}
}
\makeatother

\usepackage[font=footnotesize,labelfont=bf,tableposition=top]{caption}
\usepackage{tikz}
\usepackage{graphicx}
\usepackage[]{subcaption}

\usepackage{pgfplots}
\pgfplotsset{compat=1.6}
\usepgfplotslibrary{groupplots}
\tikzstyle{every picture}+=[font=\sffamily]
\tikzstyle{optimized} = [circle,fill=white,draw=black, dashed,inner sep=1pt, minimum size=20pt, font=\fontsize{10}{10}\selectfont, node distance=1]
\pgfkeys{/pgf/number format/.cd,1000 sep={}}
\pgfplotsset{
	tick label style = {font=\sffamily},
	every axis label/.append style={font=\sffamily},
	typeset ticklabels with strut,
}
\pgfplotsset{every axis/.append style={
			every x tick label/.append style={font=\fontsize{6pt}{6pt}\sffamily, yshift=.5ex,},
			every y tick label/.append style={font=\fontsize{6pt}{6pt}\sffamily, xshift=.5ex},
			every y label/.append style={xshift=10ex, font=\sffamily},
			every x label/.append style={yshift=3ex, font=\sffamily},
			every title/.append style={font=\sffamily}
		},
}
\pgfplotsset{
	xticklabel={$\mathsf{\pgfmathprintnumber{\tick}}$},
	yticklabel={$\mathsf{\pgfmathprintnumber{\tick}}$},
}
\pgfplotsset{every axis title/.append style={yshift=-1ex}}
\newlength\figureheight
\newlength\figurewidth
\usepgfplotslibrary{external}
\tikzexternaldisable
\newcommand*\circled[1]{\tikz[baseline=(char.base)]{
		\node[shape=circle,draw,inner sep=2pt] (char) {\tiny #1};}}

\usepackage[colorinlistoftodos,
	textsize=scriptsize,
	linecolor=red!30,
	bordercolor=red!30,
	backgroundcolor=red!10]{todonotes}
\renewcommand{\todo}[2][]{\tikzexternaldisable\@todo[#1]{#2}\tikzexternalenable}

\usepackage[most]{tcolorbox}

\tcbset{boxsep=4pt,left=2pt,right=2pt,top=-0pt,bottom=0pt}

\usepackage{url}

\usepackage{subcaption}

\usepackage{tikz}
\usepackage{graphicx}
\usepackage{xcolor}
\usepackage{amsmath}
\usepackage{amssymb}
\usepackage{bm}
\usepackage{amsthm}

\declaretheorem[name=Proposition]{proposition}
\declaretheorem[name=Definition]{definition}

\usepackage{adjustbox}
\usepackage{multirow}
\usepackage{mathtools}

\usepackage{xspace}

\newacronym{MAP}{map}{maximum-a-posteriori}
\newacronym{MLE}{mle}{maximum likelihood estimation}
\newacronym{MNLL}{mnll}{mean negative loglikelihood}
\newacronym{NLL}{nll}{negative loglikelihood}
\newacronym{LL}{ll}{log-likelihood}
\newacronym{RMSE}{rmse}{root mean square error}
\newacronym{ECE}{ece}{expected calibration error}
\newacronym{SNR}{snr}{signal-to-noise ratio}
\newacronym{FID}{fid}{Fr\'echet Inception Distance}
\newacronym{BPD}{bpd}{bit per dimension}
\newacronym{NFE}{nfe}{neural function evaluations}

\newacronym{AE}{ae}{autoencoder}
\newacronym{WAE}{wae}{Wasserstein Autoencoder}
\newacronym{VAE}{vae}{Variational Autoencoder}
\newacronym{BAE}{bae}{Bayesian autoencoder}
\newacronym{CDF}{cdf}{cumulative density function}
\newacronym{GAN}{gan}{Generative Adversarial Network}
\newacronym{DPGMM}{dpgmm}{Dirichlet process Gaussian mixture model}

\newacronym{MC}{mc}{Monte Carlo}
\newacronym{SDE}{sde}{Stochastic Differential Equation}
\newacronym{CNF}{cnf}{Continuous Normalizing Flow}
\newacronym{ODE}{ode}{Ordinary Differential Equation}
\newacronym{MCMC}{mcmc}{Markov chain Monte Carlo}
\newacronym{HMC}{hmc}{Hamiltonian Monte Carlo}
\newacronym{MH}{mh}{Metropolis-Hastings}
\newacronym{NUTS}{nuts}{no-u-turn sampler}
\newacronym{SGHMC}{sghmc}{stochastic gradient Hamiltonian Monte Carlo}

\newacronym{DGP}{dgp}{deep Gaussian process} \newacronym{GPLVM}{gplvm}{Gaussian process latent variable model}
\newacronym{DPMM}{dpmm}{Dirichlet Process Mixture Model}

\newacronym{VFE}{vfe}{variational free energy}

\newacronym[firstplural=Gaussian Processes]{GP}{gp}{Gaussian Process}

\newacronym{VI}{vi}{variational inference}

\newacronym{ELBO}{elbo}{evidence lower bound}
\newacronym{NELBO}{nelbo}{negative evidence lower bound}
\newacronym{ELL}{ell}{expected log likelihood}
\newacronym{KL}{kl}{Kullback-Leibler}
\newacronym{AUC}{auc}{area under the curve}

\newacronym[firstplural=Bayesian neural networks]{BNN}{bnn}{Bayesian neural network}
\newacronym[firstplural=deep neural networks]{DNN}{dnn}{deep neural network}
\newacronym[]{CNN}{cnn}{convolutional neural network}
\newacronym{MLP}{mlp}{multilayer perceptron}
\newacronym{NN}{nn}{neural network}
\newacronym{RELU}{ReLU}{rectified linear unit}

\newacronym{NF}{nf}{normalizing flow}

\newacronym{RBF}{rbf}{radial basis function}
\newacronym{ARD}{ard}{automatic relevance determination}

\newacronym{RKHS}{rkhs}{reproducing kernel Hilbert space}

\newacronym{OT}{ot}{optimal transport}
\newacronym{WD}{wd}{Wasserstein distance}
\newacronym{SWD}{swd}{sliced-Wasserstein distance}
\newacronym{DSWD}{dswd}{distributional sliced-Wasserstein distance}

\newcommand{\name}[1]{{\textsc{#1}}\xspace}

\newcommand{\mnist}{\name{mnist}}
\newcommand{\cifar}{\name{cifar10}}

\newcommand{\score}{\emph{score}\xspace}

\newcommand{\mathbold}[1]{{\boldsymbol{{#1}}}}

\newcommand{\g}{\,|\,}

\newcommand{\nestedmathbold}[1]{{\mathbold{#1}}}

\newcommand{\mba}{\nestedmathbold{a}}

\newcommand{\mbf}{\nestedmathbold{f}}

\newcommand{\mbm}{\nestedmathbold{m}}

\newcommand{\mbs}{\nestedmathbold{s}}

\newcommand{\mbw}{\nestedmathbold{w}}
\newcommand{\mbx}{\nestedmathbold{x}}

\newcommand{\mbI}{\nestedmathbold{I}}

\newcommand{\mbmu}{\nestedmathbold{\mu}}

\newcommand{\mbphi}{\nestedmathbold{\phi}}

\newcommand{\mbtheta}{\nestedmathbold{\theta}}

\newcommand{\mbzero}{\nestedmathbold{0}}

\newcommand{\Lelbo}{\cL_{\textsc{elbo}}}

\DeclareRobustCommand{\KL}[2]{\ensuremath{\textsc{kl}\left[#1\;\|\;#2\right]}}
\DeclarePairedDelimiterX{\infdivx}[2]{[}{]}{#1\;\delimsize\|\;#2}

\DeclareRobustCommand{\Gauss}[1]{\ensuremath{\mathcal{N}_{#1}}}

\DeclareMathOperator*{\argmax}{arg\,max}
\DeclareMathOperator*{\argmin}{arg\,min}

\newcommand{\cL}{\mathcal{L}}
\newcommand{\cN}{\mathcal{N}}

\newcommand{\cG}{\mathcal{G}}

\newcommand{\E}{\mathbb{E}}

\newcommand{\defeq}{\stackrel{\text{\tiny def}}{=}}

\newcommand{\sub}[1]{{\texttt{\textit{\scriptsize {#1}}}}}

\newcommand{\pdata}{{p_\sub{\!d\!a\!t\!a}}}
\newcommand{\ps}{{p_\sub{\!n\!o\!i\!s\!e}}}

\title{How Much is Enough? A Study on Diffusion Times in Score-based Generative Models
}

\def\eurecom{\footnotesize{EURECOM}\\\footnotesize{(France)}}
\def\huawei{\footnotesize{Huawei Technologies}\\\footnotesize{(France)}}

\author{Giulio Franzese \\\eurecom \And
	Simone Rossi \\\eurecom \And
	Lixuan Yang \\\huawei \And
	Alessandro Finamore \\\huawei \And
	Dario Rossi \\\huawei \And
	Maurizio Filippone \\\eurecom \And
	Pietro Michiardi \\\eurecom
}

\begin{document}

\maketitle

\begin{abstract}
	Score-based diffusion models are a class of generative models whose dynamics is described by stochastic differential equations that map noise into data. 
While recent works have started to lay down a theoretical foundation for these models, an analytical understanding of the role of the diffusion time $T$ is still lacking. 
Current best practice advocates for a large $T$ to ensure that the forward dynamics brings the diffusion sufficiently close to a known and simple noise distribution; however, a smaller value of $T$ should be preferred for a better approximation of the score-matching objective and higher computational efficiency.
Starting from a variational interpretation of diffusion models, in this work we quantify this trade-off, and suggest a new method to improve quality and efficiency of both training and sampling, by adopting smaller diffusion times.
Indeed, we show how an auxiliary model can be used to bridge the gap between the ideal and the simulated forward dynamics, followed by a standard reverse diffusion process.
Empirical results support our analysis; for image data, our method is competitive w.r.t. the state-of-the-art, according to standard sample quality metrics and log-likelihood.

\end{abstract}

\section{Introduction}\label{sec:introduction}
Diffusion-based generative models \citep{sohl2015deep, song2019generative, song2020score, vahdat2021score, kingma2021variational,ho2020denoising,song2021denoising} have recently gained popularity due to their ability to synthesize high-quality audio \citep{kong2021diffwave, lee2022priorgrad}, image \cite{dhariwal2021diffusion, nichol2021icml} and other data modalities \cite{Yusuke2021}, outperforming known methods based on \glspl{GAN} \citep{Goodfellow2014}, \glspl{NF} \citep{Kingma2016} or \glspl{VAE} and \glspl{BAE} \citep{Kingma14,Tran2021}.

Diffusion models learn to generate samples from an unknown density $\pdata$ by reversing a \textit{diffusion process} which transforms the distribution of interest into noise.
The forward dynamics injects noise into the data following a diffusion process that can be described by a \gls{SDE} of the form,
\begin{equation}
	\label{eq:diffusion_sde}
	\dd\mbx_t= \mbf(\mbx_t, t)\dd t+g(t)\dd\mbw_t\quad \text{with} \quad \mbx_0\sim \pdata\,,
\end{equation}
where $\mbx_t$ is a random variable at time $t$, $\mbf(\cdot, t)$ is the \textit{drift term}, $g(\cdot)$ is the \textit{diffusion term} and $\mbw_t$ is a \textit{Wiener process} (or Brownian motion). We will also consider a special class of linear \glspl{SDE}, for which the drift term is decomposed as $f(\mbx_t, t) = \alpha(t)\mbx_t$
and the diffusion term is independent of $\mbx_t$.
This class of parameterizations of \glspl{SDE} is known as \emph{affine} and it admits analytic solutions.
We denote the time-varying probability density by $p(\mbx,t)$, where by definition $p(\mbx,0)=\pdata(\mbx)$, and the conditional on the initial condition $\mbx_0$ by $p(\mbx,t\g\mbx_0)$.
The forward \gls{SDE} is usually considered for a sufficiently long \textit{diffusion time} $T$, leading to the density $p(\mbx, T)$. In principle, when $T\rightarrow\infty$, $p(\mbx, T)$ converges to Gaussian noise, regardless of initial conditions.

For generative modeling purposes, we are interested in the inverse dynamics of such process, i.e., transforming samples of the noisy distribution $p(\mbx,T)$ into $\pdata(\mbx)$.
Formally, such dynamics can be obtained by considering the solutions of the inverse diffusion process \citep{anderson},
\begin{align}\label{revsde}
	\dd\mbx_t=\left[-\mbf(\mbx_t,t')+g^2(t')\gradient\log p(\mbx_t,t')\right]\dd t +g(t')\dd\mbw_t\,, \end{align}
where $t'\defeq T-t$, with the inverse dynamics involving a new Wiener process.
Given $p(\mbx, T)$ as the initial condition, the solution of \cref{revsde}
after a \textit{reverse diffusion time} $T$, will be distributed as $\pdata(\mbx)$.
The simulation of the backward process is referred to as \textit{sampling} and,
differently from the forward process, this process is not \emph{affine} and a closed form solution is out of reach.
\paragraph{Practical considerations on diffusion time.}
In practice, diffusion models are challenging to work with \citep{song2020score}.
Indeed, a direct access to the true \score function $\gradient\log p(\mbx_t,t)$ required in the dynamics of the reverse diffusion is unavailable.
This can be solved by approximating it with a parametric function $\mbs_{\mbtheta}(\mbx_t,t)$, e.g., a neural network, which is trained using the following loss function, \begin{align}\label{eq:score_matching}
	\cL(\mbtheta) = T~\E_{p(t)} \E_{\sim\eqref{eq:diffusion_sde}} \lambda(t)\norm{\mbs_{\mbtheta}(\mbx_t,t) - \gradient\log p(\mbx_t, t \g \mbx_0)}^2\,,
\end{align}
where the notation $\E_{\sim\eqref{eq:diffusion_sde}}$ means that the expectation is taken with respect to the random process $\mbx_t$ in \cref{eq:diffusion_sde}, $p(t)=\mathcal{U}(0, T)$ and $\lambda(t)$ is a positive weighting factor.
Due to the affine property of the drift, the term $p(\mbx_t, t\g \mbx_0)$ is analytically known and normally distributed for all $t$ (expression available in \cref{tab:diff_types}, and in \citep{Sarkka2019}).
Note also that we will refer to $\lambda$ as the \emph{likelihood reweighting} factor when $\lambda(t) = g(t)^2$ \citep{song2021maximum}.
Intuitively, the estimation of the \score is akin to a denoising objective, which operates in a challenging regime. Later we will quantify precisely the difficulty of learning the \score, as a function of increasing diffusion times.

Moreover, while the forward and reverse diffusion processes are valid for all $T$,
the noise distribution $p(\mbx,T)$ is analytically known only when the diffusion time is $T\rightarrow\infty$.
To overcome this problem, the common solution is to replace $p(\mbx,T)$ with a simple distribution $\ps(\mbx)$ which, for the classes of \glspl{SDE} we consider in this work, is a Gaussian distribution.
Indeed, in the infinite diffusion time regime, it is possible to derive $p(\mbx,T\rightarrow\infty) = \ps(\mbx)$ analytically.

\begin{minipage}{.57\textwidth}
	In the literature, the discrepancy between $p(\mbx,T)$ and $\ps(\mbx)$ has been neglected, under the informal assumption of a sufficiently large diffusion time.
	Unfortunately, while this approximation seems a valid approach to simulate and generate samples, the reverse diffusion process starts from a different initial condition $q(\mbx, 0)$ and, as a consequence, it will converge to a solution $q(\mbx, T)$ that is different from the true $\pdata(\mbx)$.
	Later, we will expand on the error introduced by this approximation, but for illustration purposes \cref{fig:toy_difftime_vs_likelihood} shows quantitatively this behavior for a simple 1D toy example $\pdata(\mbx)=\pi\cN(1, 0.1^2) + (1-\pi)\cN(3, 0.5^2)$, with $\pi=0.3$: when $T$ is small, the distribution $\ps(\mbx)$ is very different from $p(\mbx, T)$ and samples from $q(\mbx, T)$ exhibit very low likelihood of being generated from $\pdata(\mbx)$.
	\end{minipage}\hfill
\begin{minipage}{.4\textwidth}
	\centering
	\includegraphics[width=\textwidth]{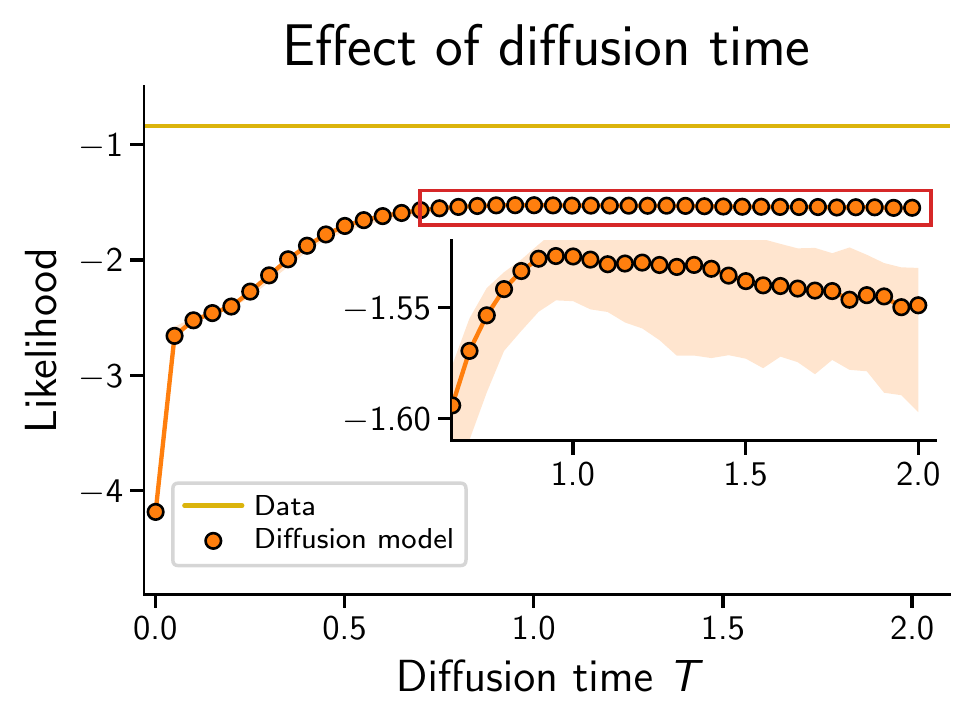}
	\captionof{figure}{Effect of $T$ on a toy model: low diffusion times are detrimental for sample quality (likelihood of 1024 samples as median and 95 quantile, on 8 random seeds).
		}
	\label{fig:toy_difftime_vs_likelihood}
	\end{minipage}

Crucially, \cref{fig:toy_difftime_vs_likelihood} (zoomed region) illustrates an unknown behavior of diffusion models, which we unveil in our analysis. In practical settings, there exists an optimal diffusion time that strikes right the balance between efficient \score estimation, and sampling quality.

\paragraph{Contributions.}
An appropriate choice of the diffusion time $T$ is a key factor that impacts training convergence, sampling time and quality.
On the one hand, the approximation error introduced by considering initial conditions for the reverse diffusion process drawn from a simple distribution $\ps(\mbx) \neq p(\mbx,T)$ increases when $T$ is small.
This is why the current best practice is to choose a sufficiently long diffusion time.
On the other hand, training convergence of the \score model $\mbs_{\mbtheta}(\mbx_t,t)$ becomes more challenging to achieve with a large $T$, which also imposes extremely high computational costs both for training and for sampling.
This would suggest to choose a smaller diffusion time.
Given the importance of this problem, in this work we set off to study---for the first time---the existence of suitable operating regimes to strike the right balance between computational efficiency and model quality. The main contributions of this work are the following.

\noindent \textbf{Contribution 1:} In \cref{sec:elbo} we provide a new characterization of score-based diffusion models, which allows us to obtain a formal understanding of the impact of the diffusion time $T$.
We do so by introducing a novel decomposition of the \gls{ELBO}, which emphasizes the roles of (i) the discrepancy between the ``ending'' distribution of the diffusion and the ``starting'' distribution of the reverse diffusion processes, and (ii) of the \score matching objective.
This allows us to claim the existence of an optimal diffusion time, and it provides, for the first time, a formal assessment of the current best practice for selecting $T$.

\noindent \textbf{Contribution 2:} In \cref{sec:method} we propose a novel method to improve \emph{both} training and sampling efficiency of diffusion-based models, while maintaining high sample quality.
Our method introduces an auxiliary distribution, allowing us to transform the simple ``starting'' distribution of the reverse process used in the literature so as to minimize the discrepancy to the ``ending'' distribution of the forward process.
Then, a standard reverse diffusion can be used to closely match the data distribution.
Intuitively, our method allows to build ``bridges'' across multiple distributions, and to set $T$ toward the advantageous regime of small diffusion times.

In addition to our methodological contributions, in \cref{sec:experiments}, we provide experimental evidence of the benefits of our method, in terms of sample quality and log likelihood.
Finally, we conclude in \cref{sec:conclusion}.

\paragraph{Related Work.}
A concurrent work \cite{zheng2022truncated} presents an empirical study of a truncated diffusion process, but lacks a rigorous analysis, and a clear justification for the proposed approach. Attempts \cite{lee2022priorgrad} to optimize $\ps$, or the proposal to do so \cite{austin2021structured} have been studied in different contexts. Related work focus primarily on improving sampling efficiency, using a wide array of techniques. Sample generation times can be drastically reduced considering adaptive step-size integrators \cite{jolicoeur2021gotta}. Other popular choices are based on merging multiple steps of a pretrained model through distillation techniques \cite{salimans2022progressive} or by taking larger sampling steps with GANs \cite{xiao2022DDGAN}. 
Approaches closer to ours \textit{modify} the \gls{SDE}, or the discrete time processes, to obtain inference efficiency gains. In particular \cite{song2021denoising} considers implicit non-Markovian diffusion processes, while \cite{watson2021} changes the diffusion processes by optimal scheduling selection and \cite{dockhorn2022score} considers overdamped \gls{SDE}s. Finally, hybrid techniques combining VAEs and diffusion models \citep{vahdat2021score}, have positive effects on training and sampling times.

\section{A new ELBO decomposition and a tradeoff on diffusion time}\label{sec:elbo}

\glsreset{ELBO}
The dynamics of a diffusion model can be studied through the lens of variational inference, which allows us to bound the (log-)likelihood using an \gls{ELBO} \citep{huang2021variational}.
Our interpretation emphasizes the two main factors affecting the quality of sample generation: an imperfect \score, and a mismatch, measured in terms of the \gls{KL} divergence, between the noise distribution $p(\mbx,T)$ of the forward process and the distribution $\ps$ used to initialize the backward process.

\subsection{The ELBO decomposition}

By manipulating the $\Lelbo$ derived in \citep[Eq. (25)]{huang2021variational}, we can write
\begin{flalign}\label{elbo0}
	&\mathbb{E}_{\pdata(\mbx)}\log q(\mbx,T)\geq \Lelbo(\mbs_\mbtheta, T) = \E_{\sim\eqref{eq:diffusion_sde}}\log \ps(\mbx_T)-I(\mbs_{\mbtheta},T)+R(T),
\end{flalign}
where 
$R(T)=\frac{1}{2}\int\limits_{t=0}^{T}\E_{\sim\eqref{eq:diffusion_sde}}\left[g^2(t)\norm{\gradient\log p(\mbx_t,t\g\mbx_0)}^2-2\mbf^\top(\mbx_t,t)\gradient\log p(\mbx_t,t\g\mbx_0)\right]\dd t$,
and
$I(\mbs_{\mbtheta},T)= \frac{1}{2}\int\limits_{t=0}^{T}g^2(t)\E_{\sim\eqref{eq:diffusion_sde}}
	\left[  \norm{\mbs_{\mbtheta}(\mbx_t,t)-\gradient\log p(\mbx_t,t\g\mbx_0)}^2\right]\dd t$. 
Note that $R(T)$ depends neither on $\mbs_{\mbtheta}$ nor on $\ps$, while $I(\mbs_{\mbtheta},T)$, or an equivalent reparameterization \citep[\cref{eq:diffusion_sde}]{huang2021variational,song2021maximum}, is used to learn the approximated \score, by optimization of the parameters $\mbtheta$.
It is then possible to show that
\begin{equation}\label{eq:I_ineq}
	I(\mbs_{\mbtheta},T)\geq \underbrace{I(\gradient\log p,T)}_{\defeq K(T)}=\frac{1}{2}\int\limits_{t=0}^{T}g^2(t)\mathbb{E}_{\sim\eqref{eq:diffusion_sde}}\left[
		\norm{\gradient\log p(\mbx_t,t)-\gradient\log p(\mbx_t,t\g\mbx_0)}
		\right]^2\dd t.
\end{equation}
Consequently, we can rewrite $I(\mbs_{\mbtheta},T))=K(T)+\cG(\mbs_{\mbtheta},T)$ (see Appendix for details), where $\cG(\mbs_{\mbtheta},T)$ is a positive term that we call the \textit{gap} term, accounting for the practical case of  an imperfect \score, i.e. $\mbs_{\mbtheta}(\mbx_t,t) \neq \gradient\log p(\mbx_t,t)$.
It also holds that
\begin{align}
	\E_{\sim\eqref{eq:diffusion_sde}}\log \ps(\mbx_T) & =  \int \left[ \log \ps(\mbx) - \log p(\mbx, T) + \log p(\mbx, T) \right] p(\mbx, T) \dd\mbx = \nonumber \\
	                                                  & =  \E_{\sim\eqref{eq:diffusion_sde}}\log p(\mbx_T, T) -\KL{\log p(\mbx, T)}{\ps(\mbx)}.
\end{align}
Therefore, we can rewrite the \gls{ELBO} in \cref{elbo0} as
\begin{equation}
	\mathbb{E}_{\pdata(\mbx)}\log q(\mbx,T) \geq
	- \KL{p(\mbx,T)}{\ps(\mbx)} + \mathbb{E}_{\sim\eqref{eq:diffusion_sde}}\log p(\mbx_T,T) - K(T) + R(T) - \cG(\mbs_{\mbtheta},T).
\end{equation}

Before concluding our derivation it is necessary to introduce an important observation (formal proof in Appendix).

\begin{restatable}{proposition}{propelbo}\label{prop_elbo}
	Given the stochastic dynamics defined in \cref{eq:diffusion_sde}, it holds that
	\begin{align}
		\mathbb{E}_{\sim\eqref{eq:diffusion_sde}}\log p(\mbx_T,T) - K(T) + R(T) = \mathbb{E}_{\pdata(\mbx)}\log \pdata(\mbx).
	\end{align}
\end{restatable}

Finally, we can now bound the value of $\mathbb{E}_{\pdata(\mbx)}\log q(\mbx,T)$ as
\begin{equation}\label{elbo1}
	\mathbb{E}_{\pdata(\mbx)}\log q(\mbx,T) \geq
	\underbrace{\mathbb{E}_{\pdata(\mbx)}\log \pdata(\mbx) - \cG(\mbs_{\mbtheta},T) - \KL{p(\mbx,T)}{\ps(\mbx)}}_{\Lelbo(\mbs_\mbtheta, T)}.
\end{equation}
\cref{elbo1} clearly emphasizes the roles of an approximate score function, through the gap term $\cG(\cdot)$, and the discrepancy between the noise distribution of the forward process, and the initial distribution of the reverse process, through the \gls{KL} term.
In the ideal case of perfect \score matching, the \gls{ELBO} in \cref{elbo1} is attained with equality.
If, in addition, the initial conditions for the reverse process are ideal, i.e. $q(\mbx,0)=p(\mbx,T)$, then the results in \cite{anderson} allow us to claim that $q(\mbx,T)=\pdata(\mbx)$.

Next, we show the existence of a tradeoff: the \gls{KL} decreases with $T$, while the gap increases with $T$.

\subsection{The tradeoff on diffusion time}\label{subsec:tradeoff}

We begin by showing that the \gls{KL} term in \cref{elbo1} decreases with the diffusion time $T$, which induces to select large $T$ to maximize the \gls{ELBO}.
We consider the two main classes of \glspl{SDE} for the forward diffusion process defined in \cref{eq:diffusion_sde}: \glspl{SDE} whose steady state distribution is the standard multivariate Gaussian, referred to as \emph{Variance Preserving} (VP), and \glspl{SDE} without a stationary distribution, referred to as \emph{Variance Exploding} (VE), which we summarize in \cref{tab:diff_types}.
The standard approach to generate new samples relies on the backward process defined in \cref{revsde}, and consists in setting $\ps$ in agreement with the form of the forward process \gls{SDE}.
The following result bounds the discrepancy between the noise distribution $p(\mbx,T)$ and $\ps$.

\begin{table}[t]
	\centering\scriptsize\sffamily
	\setlength{\tabcolsep}{4pt}
	\caption{Two main families of diffusion processes, where  $\sigma^2(t) = \left(\frac{\sigma^2_{\sub{\tiny max}}}{\sigma^2_{\sub{\tiny min}}}\right)^t$ and $\beta(t) = \beta_0 + (\beta_1 - \beta_0)t$}
	\label{tab:diff_types}
	\begin{tabular}{rcccc}
		\toprule[1pt]
		                    & Diffusion process                                            & $p(\mbx_t, t\g\mbx_0) = \cN(\mbm,s\mbI)$                           & $\ps(\mbx)$                                        \\
		\midrule
		Variance Exploding  & $\alpha(t) = 0$, $g(t)=\sqrt{\frac{\dd\sigma^2(t)}{\dd t}} $ & $\mbm = \mbx_0$, $s=\sigma^2(t) - \sigma^2(0)$                     & $\mathcal{N}(\mbzero,\sigma^2(T)-\sigma^2(0)\mbI)$ \\

		Variance Preserving & $\alpha(t) = -\frac{1}{2}\beta(t)$, $g(t)=\sqrt{\beta(t)} $  & $\mbm= e^{-\frac{1}{2}\int_0^t\beta(\dd\tau)}\mbx_0$, $s = 1-e^{-\int_0^t\beta(\dd\tau)} $ & $\mathcal{N}(\mbzero,\mbI)$                        \\
		\bottomrule[1pt]
	\end{tabular}
\end{table}

\begin{restatable}{lemma}{propklvanish}\label{prop:kl_vanish}
	For the classes of \glspl{SDE} considered (\cref{tab:diff_types}), the discrepancy between $p(\mbx,T)$ and the $\ps(\mbx)$ can be bounded as follows.

	For Variance Preserving \glspl{SDE},
	it holds that: $\KL{p(\mbx,T)}{\ps(\mbx)} \leq C_1\exp(-\int_{0}^{T}\beta(t)dt)$.
	
	For Variance Exploding \glspl{SDE},
	it holds that: $\KL{p(\mbx,T)}{\ps(\mbx)} \leq C_2\frac{1}{\sigma^2(T)-\sigma^2(0)}$.
	\end{restatable}

Our proof uses results from \citep{villani2009optimal}, the logaritmic Sobolev Inequality and Gronwall inequality (see Appendix for details).
The consequence of \cref{prop:kl_vanish} is that to maximize the \gls{ELBO}, the diffusion time $T$ should be as large as possible (ideally, $T \to \infty$), such that the \gls{KL} term vanishes.
This result is in line with current practices for training score-based diffusion processes, that argue for sufficiently long diffusion times \cite{de2021diffusion}.
Our analysis, on the other hand, highlights how this term is only one of the two contributions to the \gls{ELBO}.

Now, we focus our attention on studying the behavior of the second component, $\cG(\cdot)$.
Before that, we define a few quantities that allow us to write the next important result.

\begin{definition}\label{def:optimal_score}
	We define the \underline{optimal score} $\widehat\mbs_\mbtheta$ for {any} diffusion time $T$, as the score obtained using parameters that minimize $I(\mbs_{\mbtheta},T)$.
	Similarly, we define the \underline{optimal score gap} $\cG(\widehat\mbs_{\mbtheta},T)$ for {any} diffusion time $T$, as the gap attained when using the {optimal score}.
\end{definition}

\begin{restatable}{lemma}{propscoreerror}\label{prop:score_error}
	The optimal score gap term ${\cG}(\widehat\mbs_{\mbtheta},T)$ is a non-decreasing function in $T$.
	That is, given $T_2 > T_1$, and $\mbtheta_1=\argmin_{\mbtheta}I(\mbs_{\mbtheta},T_1),\mbtheta_2=\argmin_{\mbtheta}I(\mbs_{\mbtheta},T_2)$,
	then $\cG(\mbs_{\mbtheta_2},T_2) \geq \cG(\mbs_{\mbtheta_1},T_1)$.
	\end{restatable}

The proof (see Appendix) is a direct consequence of the definition of $\cG$ and the optimality of the score.
Note that \cref{prop:score_error} does not imply that $\cG(\mbs_{\mbtheta_a},T_2)\geq\cG(\mbs_{\mbtheta_b},T_1)$ holds for generic parameters $\mbtheta_a,\mbtheta_b$.
Moreover, although it would be tempting to use a pre-trained score function for a given diffusion time $T_2$, and apply it in a reverse process for a different diffusion time $T_1$, this would be detrimental.
Indeed, given $\mbtheta_1, \mbtheta_2$ from \cref{prop:score_error}, $\cG(\mbs_{\mbtheta_2},T_1) \geq \cG(\mbs_{\mbtheta_1},T_1)$. As $\mbtheta_1$ is optimal for $T_1$ anything else is consequently sub-optimal. Practically speaking, this means that if the goal is to fit a score function to be used for time $T$, then the training should be performed for the same $T$.

\subsection{Is there an optimal diffusion time?}
While diffusion processes are generally studied for $T \rightarrow \infty$, for practical reasons, diffusion times in score-based models have been arbitrarily set to be ``sufficiently large'' in the literature.
Here we formally argue, for the first time, about the existence of an optimal diffusion time, which strikes the right balance between the gap $\cG(\cdot)$ and the \gls{KL} terms of the \gls{ELBO} in \cref{elbo1}.

\begin{restatable}{proposition}{theotstar}\label{theo_tstar}
	There exists at least one optimal diffusion time $T^\star$ in the interval $[0, \infty]$, which maximizes the \gls{ELBO}, that is $T^\star = \argmax_T \Lelbo(\widehat\mbs_\theta, T)$.
	\end{restatable}

While the proof for the general case is available in the Appendix, the analytic solution for the optimal diffusion time is elusive, as a full characterization of the gap term is particularly challenging.
Additional assumptions would guarantee boundedness of $T^\star$.

\begin{figure}[htpb]
	\centering
	\includegraphics[width=0.47\textwidth]{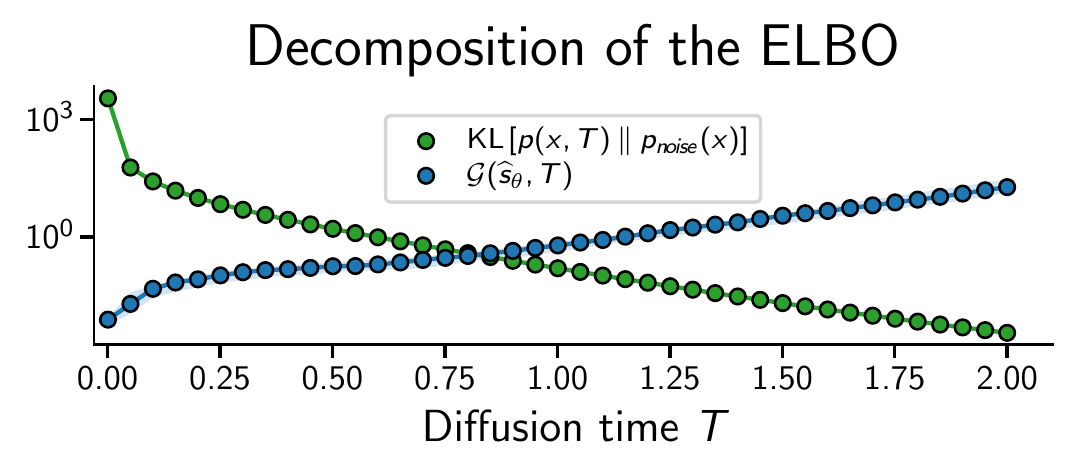}
	\includegraphics[width=0.47\textwidth]{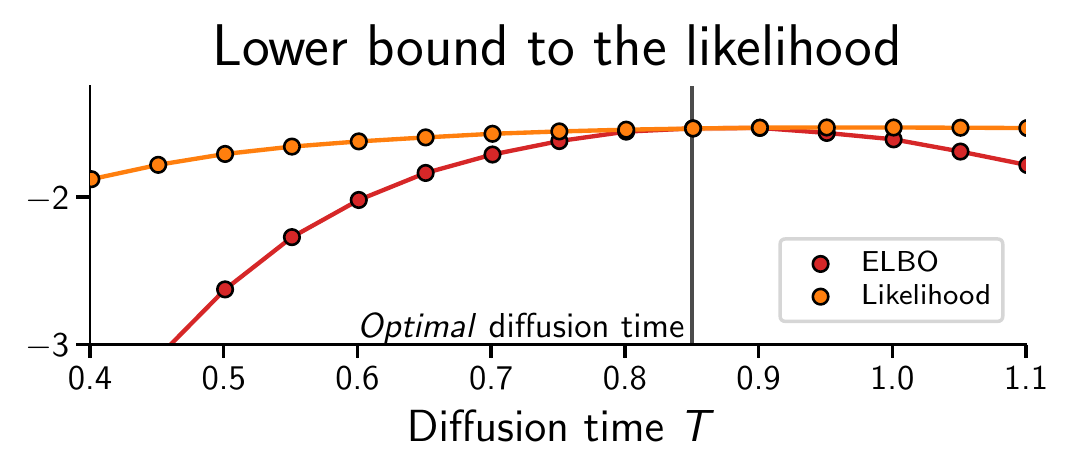}
	\caption{\gls{ELBO} decomposition, \gls{ELBO} and likelihood for a 1D toy model, as a function of diffusion time $T$. Tradeoff and optimality numerical results confirm our theory.}
	\label{fig:toy_elbo_decomposition}
\end{figure}

Empirically, we use \cref{fig:toy_elbo_decomposition} to illustrate the tradeoff and the optimality arguments through the lens of the same toy example we use in \cref{sec:introduction}.
On the left, we show the \gls{ELBO} decomposition.
We can verify that $\cG(\mbs_\mbtheta, T)$ is an increasing function of $T$, whereas the \textsc{kl} term is a decreasing function of $T$.
Even in the simple case of a toy example, the tension between small and large values of $T$ is clear.
On the right, we show the values of the \gls{ELBO} and of the likelihood as a function of $T$.
We then verify the validity of our claims: the \gls{ELBO} is neither maximized by an infinite diffusion time, nor by a ``sufficiently large'' value. Instead, there exists an optimal diffusion time $T^\star \eqsim 0.85$ which, for this example, is smaller than what is typically used in practical implementations, i.e. $T=1.0$

Next, we present a new method that admits smaller diffusion times, which is the ultimate goal of our work.
We also show that the \gls{ELBO} of our approach is at least as good as the one of a standard diffusion model, configured to use its optimal diffusion time $T^\star$.

\section{A practical new method for decreasing diffusion times}\label{sec:method}
The \gls{ELBO} decomposition in \cref{elbo1} and the bounds in \cref{prop:kl_vanish} and \cref{prop:score_error} highlight a dilemma. 
We thus propose a simple method that allows us to achieve both a small gap $\cG(\mbs_{\mbtheta},T)$, and a small discrepancy $\KL{p(\mbx,T)}{\ps(\mbx)}$. 
Before that, let us use \cref{fig:bridges} to summarize all densities involved and the effects of the various approximations, which will be useful to visualize our proposal.

\definecolor{tikz_red}{HTML}{D0021B}
\definecolor{tikz_green}{HTML}{7ED321}
\definecolor{tikz_orange}{HTML}{F5A623}
\definecolor{tikz_cyan}{HTML}{4A90E2}

\begin{minipage}{.46\textwidth}
The data distribution $\pdata(\mbx)$ is transformed into the noise distribution $p(\mbx,T)$ through the forward diffusion process. 
\textcolor{tikz_green!80!black}{Ideally}, starting from $p(\mbx,T)$ we can recover the data distribution by simulating using the exact score $\gradient\log p$. 
\textcolor{tikz_orange!80!black}{Using the approximated score} $\mbs_{\mbtheta}$ and the same initial conditions, the backward process ends up in $q^{(1)}(\mbx,T)$, whose discrepancy \circled{1} to $\pdata(\mbx)$ is $\cG(\mbs_{\mbtheta},T)$.
However, the distribution $p(\mbx,T)$ is unknown and replaced with an easy distribution $\ps(\mbx)$, accounting for an error \circled{a} measured as $\KL{p(\mbx,T)}{\ps(\mbx)}$.
\textcolor{tikz_red!80!black}{With score and initial distribution approximated}, the backward process ends up in $q^{(3)}(\mbx,T)$, where the discrepancy \circled{3} from $\pdata$ is the sum of terms $\cG(\mbs_{\mbtheta},T)+\KL{p(\mbx,T)}{\ps}$.
\end{minipage}\hfill
\begin{minipage}{.52\textwidth}
\centering
	\scalebox{0.6}{

\input{bridges_3}
		}
	
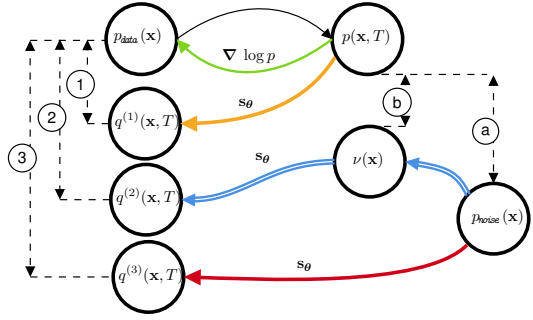
\captionof{figure}{Intuitive illustration of the forward and backward diffusion processes. Discrepancies between distributions are illustrated as distances. Color coding discussed in the text.
}
	\label{fig:bridges}
\end{minipage}

\paragraph{Multiple bridges across densities.}
In summary, we reduce the gap term by selecting smaller diffusion times and we propose to \textcolor{tikz_cyan!80!black}{learn an auxiliary model} to transform the initial density $\ps(\mbx)$ into a density $\nu_{\mbphi}(\mbx)$, which is as close as possible to $p(\mbx,T)$, thus avoiding the penalty of a large \gls{KL} term. 
To implement this, we first \textit{transform} the simple distribution $\ps$ into the distribution $\nu_{\mbphi}(\mbx)$, whose discrepancy \circled{b} $\KL{p(\mbx,T)}{\nu_{\mbphi}(\mbx)}$ is smaller than \circled{a} . Then, \textcolor{tikz_cyan!80!black}{starting from from the auxiliary model} $\nu_{\mbphi}(\mbx)$, we use the approximate score $\mbs_{\mbtheta}$ to simulate the backward process reaching $q^{(2)}(\mbx,T)$. 
This solution has a discrepancy \circled{2} from the data distribution of $\cG(\mbs_{\mbtheta},T)+\KL{p(\mbx,T)}{\nu_{\mbphi}(\mbx)}$, which we will quantify later in the section. 
Intuitively, we introduce two bridges. 
The first bridge connects the noise distribution $\ps$ to an auxiliary distribution $\nu_{\mbphi}(\mbx)$ that is as close as possible to that obtained by the forward diffusion process. 
The second bridge---a standard reverse diffusion process---connects the smooth distribution $\nu_{\mbphi}(\mbx)$ to the data distribution. 
Notably, our approach has important guarantees, which we discuss next.

\subsection{Auxiliary model fitting and guarantees}\label{sec:firstbridge}

We begin by stating the requirements we consider for the density $\nu_{\mbphi}(\mbx)$.
First, as it is the case for $\ps$, it should be easy to generate samples from $\nu_{\mbphi}(\mbx)$ in order to initialize the reverse diffusion process. 
Second, the auxiliary model should allow us to compute the likelihood of the samples generated through the overall generative process, which begins in $\ps$, passes through $\nu_{\mbphi}(\mbx)$, and arrives in $q(\mbx, T)$.

The fitting procedure of the auxiliary model is straightforward. 
First, we recognize that minimizing $\KL{p(\mbx,T)}{\nu_{\mbphi}(\mbx)}$ w.r.t $\mbphi$ also minimizes $\mathbb{E}_{p(\mbx,T)}\left[\log\nu_{\mbphi}(\mbx)\right]$, that we can use as loss function. 
To obtain the set of optimal parameters $\mbphi^\star$, we require samples from $p(\mbx,T)$, which can be easily obtained even if the density $p(\mbx,T)$ is not available.
Indeed, by sampling from $\pdata$, and $p(\mbx,T\g\mbx_0)$, we obtain an unbiased Monte Carlo estimate of $\mathbb{E}_{p(\mbx,T)}\left[\log\nu_{\mbphi}(\mbx)\right]$, and optimization of the loss can be performed.
Note that due to the affine nature of the drift, the conditional distribution $p(\mbx,T\g\mbx_0)$ is easy to sample from, as shown in \cref{tab:diff_types}.
From a practical point of view, it is important to notice that the fitting of $\nu_\phi$ is independent from the training of the score-matching objective, i.e. the result of $I(\mbs_\mbtheta)$ does not depend on the shape of this auxiliary distribution $\nu_\phi$.
This observation indicates that the two training procedures can be run concurrently, thus enabling considerable time savings.

Next, we show that the first bridge in our model reduces the \gls{KL} term, even for small diffusion times.
\begin{proposition}\label{prop:klless} 
Let's assume that $\ps(\mbx)$ is in the family spanned by $\nu_{\mbphi}$, i.e. there exists $\widetilde\mbphi$ such that $\nu_{\widetilde\mbphi}=\ps$.
Then we have that 
	\begin{equation}
		\KL{p(\mbx,T)}{\nu_{\mbphi^*}(\mbx)} \leq
		\KL{p(\mbx,T)}{\nu_{\widetilde\mbphi}(\mbx) } = 
		\KL{p(\mbx,T)}{\ps(\mbx)}.
	\end{equation}
\end{proposition}

Since we introduce the auxiliary distribution $\nu$, we shall define a new \gls{ELBO} for our method:
\begin{align}
\Lelbo^{\mbphi}(\mbs_\mbtheta, T)=\mathbb{E}_{\pdata(\mbx)}\log \pdata(\mbx) - \cG(\mbs_{\mbtheta},T) - \KL{p(\mbx,T)}{\nu_{\mbphi}(\mbx)}
\end{align}

Recalling that $\widehat\mbs_\mbtheta$ is the optimal score for a generic time $T$, \cref{prop:klless} allows us to claim that $\Lelbo^{\mbphi^\star}(\widehat\mbs_\mbtheta, T) \geq \Lelbo(\widehat\mbs_\mbtheta, T)$. 
Then, we can state the following important result:
\begin{restatable}{proposition}{proptstarstar}\label{prop_tstarstar} 
Given the existence of $T^\star$, defined as the diffusion time such that the \gls{ELBO} is maximized (\cref{theo_tstar}), there exists at least one diffusion time 
$\tau \in [0,T^\star]$, such that 
$\Lelbo^{\mbphi^\star}(\widehat\mbs_\mbtheta, \tau) \geq \Lelbo(\widehat\mbs_\mbtheta, T^{*})$.
\end{restatable}

\Cref{prop_tstarstar} has two interpretations. On the one hand, given two score models optimally trained for their respective diffusion times, our approach guarantees an \gls{ELBO} that is at least as good as that of a standard diffusion model configured with its optimal time $T^\star$. Our method achieves this with a smaller diffusion time $\tau$, which offers sampling efficiency and generation quality. On the other hand, if we settle for an equivalent \gls{ELBO} for the standard diffusion model and our approach, with our method we can afford a sub-optimal score model, that requires a smaller computational budget to be trained, while guaranteeing shorter sampling times. We elaborate on this interpretation in \cref{sec:experiments}, where our approach obtains substantial savings in terms of training iterations.

A final note is in order. The choice of the auxiliary model depends on the selected diffusion time. 
The larger the $T$, the ``simpler'' the auxiliary model can be. 
Indeed, the noise distribution $p(\mbx,T)$ approaches $\ps$, so that a simple auxiliary model is sufficient to transform $\ps$ into a distribution $\nu_{\mbphi}$. 
Instead, for a small $T$, the distribution $p(\mbx,T)$ is closer to the data distribution. 
Then, the auxiliary model requires high flexibility and capacity.
In \cref{sec:experiments}, we substantiate this discussion with numerical examples and experiments on real data.

\subsection{An extension for density estimation}\label{sec:densityest}

Diffusion models can be also used for density estimation by transforming the diffusion \gls{SDE} into an equivalent \gls{ODE}
whose marginal distributions $p(\mbx,t)$
at each time instant coincide to that of the corresponding \gls{SDE} \cite{song2020score}.
The exact equivalent \gls{ODE} requires the score $\gradient\log p(\mbx_t,t)$, which in practice is replaced by the score model $\mbs_\mbtheta$, leading to the following \gls{ODE}
\begin{equation}\label{eq:odeflow}
	\dd\mbx_t=\left(\mbf(\mbx_t,t)-\frac{1}{2}g(t)^2\mbs_{\mbtheta}(\mbx_t,t)\right)\dd t\quad \text{with} \quad \mbx_0\sim \pdata\,,
\end{equation}
whose time varying probability density is indicated with $\widetilde{p}(\mbx,t)$.
Note that the density $\widetilde p(\mbx,t)$, is in general \textbf{not} equal to the density $p(\mbx,t)$ associated to \cref{eq:diffusion_sde}, with the exception of perfect score matching \cite{song2021maximum}. 
The reverse time process is modeled as a \gls{CNF} \cite{chen2018cnf, grathwohl2018scalable} initialized with distribution $\ps(\mbx)$; then, the likelihood of a given value $\mbx_0$ is
\begin{align}
	\log \widetilde{p}(\mbx_0)= \log\ps(\mbx_T)+\int\limits_{t=0}^{T}\div\left(\mbf(\mbx_t,t)-\frac{1}{2}g(t)^2\mbs_{\mbtheta}(\mbx_t,t)\right)\dd t.
\end{align}

For our proposed model, we also need to take into account this new \gls{ODE} dynamics.
We focus again on the term $\log\ps(\mbx_T)$ to improve the expected log likelihood. For consistency, our auxiliary density $\nu_{\mbphi}$ should now maximize $\mathbb{E}_{\sim \eqref{eq:odeflow}}\log\nu_{\mbphi}({\mbx_T})$ instead of $\mathbb{E}_{\sim \eqref{eq:diffusion_sde}}\log\nu_{\mbphi}({\mbx_T})$. However, the simulation of \cref{eq:odeflow} requires access to $\mbs_{\mbtheta}$ which, in the endeavor of density estimation, is available only once the score model has been trained.
Consequently optimization w.r.t. $\mbphi$ can only be performed sequentially, whereas for generative purposes it could be done concurrently.
While the sequential version is expected to perform better, experimental evidence indicates that improvements are marginal, justifying the adoption of the more efficient concurrent version.

\section{Experiments}\label{sec:experiments}
We now present numerical results on the \mnist and \cifar datasets, to support our claims in \cref{sec:elbo,sec:method}.
We follow a standard experimental setup \citep{song2021denoising,song2021maximum,huang2021variational,kingma2021variational}:
we use a standard U-Net architecture with time embeddings \citep{ho2020denoising} and we report the log-likelihood in terms of \gls{BPD} and the \gls{FID} scores (uniquely for \cifar). Although the \gls{FID} score is a standard metric for ranking generative models, caution should be used against over-interpreting \gls{FID} improvements \cite{Kynkaanniemi2022FID}. Similarly, while the theoretical properties of the models we consider are obtained through the lens of \gls{ELBO} maximization, the log-likelihood measured in terms of \gls{BPD} should be considered with care \cite{theis2016bpd}.
Finally, we also report the number of \gls{NFE} for computing the relevant metrics. We compare our method to the standard score-based model \cite{song2020score}.
Training and evaluation is performed on a small cluster with 16 NVIDIA V100 GPUs.
The full description on the experimental setup is presented in Appendix.

\begin{minipage}{.6\textwidth}
	\paragraph{On the existence of $T^\star$.}
	We look for further empirical evidence of the existence of a $T^\star<\infty$, as stated in \cref{theo_tstar}.
	For the moment, we shall focus on the baseline model \citep{song2020score}, where no auxiliary models are introduced.
	Results are reported in \cref{tab:existence_T_baseline}.
	For \mnist, we observe how times $T=0.6$ and $T=1.0$ have comparable performance in terms of \gls{BPD}, implying that any $T\geq1.0$ is at best unnecessary and generally detrimental.
	Similarly, for \cifar, it is possible to notice  that the best value of \gls{BPD} is achieved for $T=0.6$, outperforming all other values.
	\end{minipage}
\hfill
\begin{minipage}{.35\textwidth}
	\centering
	\captionof{table}{Optimal $T$ in \citep{song2020score}}
	\label{tab:existence_T_baseline}
	\scriptsize
\setlength\tabcolsep{4pt}
\setlength\extrarowheight{-4pt}
\begin{tabular}{rccccc}
	\toprule[1pt]
	\textbf{Dataset} & \textbf{Time $T$} & \textbf{\textsc{bpd}} ($\downarrow$) \\
	\midrule
	\mnist         & $1.0$             & $1.16$                               \\
	               & $0.6$             & $\mathbf{1.16}$                      \\
	               & $0.4$             & $1.25$                               \\
	               & $0.2$             & $1.75$                               \\
	\midrule
	\cifar         & $1.0$             & $3.09$                               \\
	               & $0.6$             & $\mathbf{3.07}$                      \\
	               & $0.4$             & $3.09$                               \\
	               & $0.2$             & $3.38$                               \\
	\bottomrule[1pt]
\end{tabular}

\end{minipage}

\paragraph{Our auxiliary models.}
In \cref{sec:method} we introduced an auxiliary model to minimize the mismatch between initial distributions of the backward process. We now specify the family of parametric distributions we have considered.
Clearly, the choice of an auxiliary model also depends on the data distribution, in addition to the choice of diffusion time $T$.

\begin{minipage}{.58\textwidth}
	For our experiments, we consider two auxiliary models: (i) a \gls{DPGMM} \citep{Rasmussen1999a,Gorur2010} for \mnist and (ii) Glow  \cite{kingma2018glow}, a highly flexible normalizing flow for \cifar.
	Both of them satisfy our requirements: they allow for exact likelihood computation and they are equipped with a simple sampling procedure.
	As discussed in \cref{sec:method}, auxiliary model complexity should be adjusted as a function of $T$.
	This is confirmed experimentally in \cref{fig:aux_complexity}, where we use the number of mixture components of the \gls{DPGMM} as a proxy to measure the complexity of the auxiliary model.
\end{minipage}\hfill
\begin{minipage}{.4\textwidth}
	\centering
	\includegraphics[width=\textwidth]{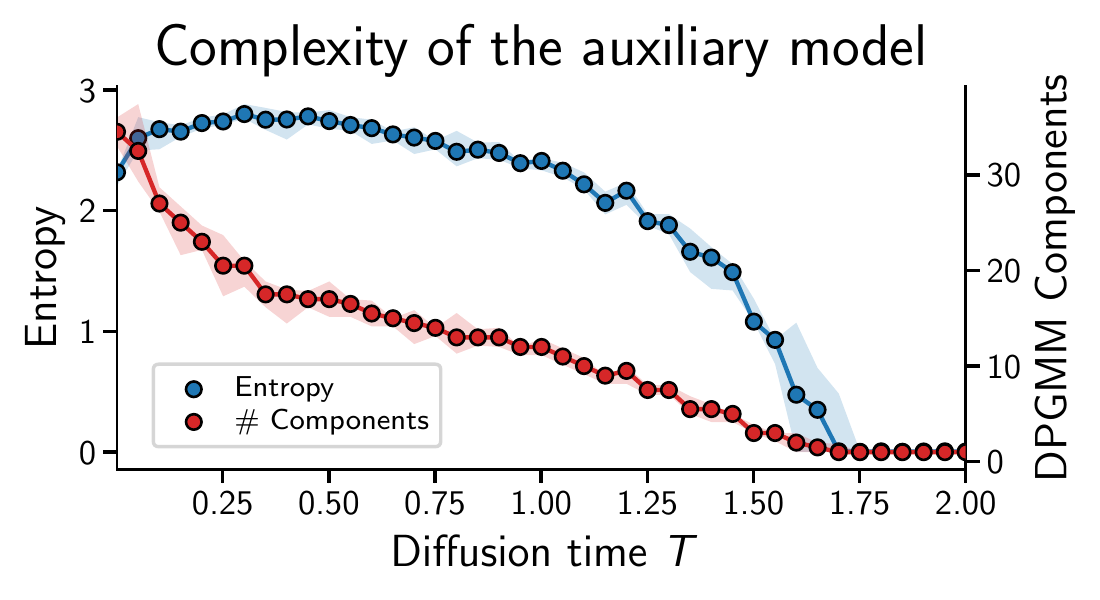}
	\captionof{figure}{Complexity of the auxiliary model as function of diffusion time (reported median and 95 quantiles on 4 random seeds).}
	\label{fig:aux_complexity}
\end{minipage}

\begin{figure}[t]
	\centering
	\begin{minipage}{0.56\textwidth}
		\centering
		\captionof{table}{Experimental results on \cifar, including other relevant baselines and sampling efficiency enhancements from the literature.
			}
		\label{tab:cifar10}
		
\setlength{\tabcolsep}{4pt}
\scriptsize
\begin{tabular}{r|rrrr}
	\toprule[1pt]
	                                                          & \textbf{\textsc{fid}}$(\downarrow)$ & \textbf{\textsc{bpd}} $(\downarrow)$ & \textbf{\textsc{nfe}} $(\downarrow)$ & \textbf{\textsc{nfe}} $(\downarrow)$ \\
	\textbf{Model}                                            &                                     &                                      & (\gls{SDE})\;\;\;\;                  & (\gls{ODE})\;\;\;\;                  \\
	\midrule
	ScoreSDE \cite{song2020score}                             & $3.64 $                             & $3.09$                               & $1000$    & $221$                                                                  \\
	\midrule
	\rowcolor{Gray!20} ScoreSDE  ($T=0.6$)                    & $5.74$                              & $3.07$                               & $600 $                               & $200$                                \\
	\rowcolor{Gray!40} ScoreSDE  ($T=0.4$)                    & $24.91$                             & $3.09$                               & $400 $                               & $187$                                \\
	\rowcolor{Gray!60} ScoreSDE  ($T=0.2$)                    & $339.72$                            & $3.38$                               & $200 $                               & $176$                                \\
	\midrule
	\rowcolor{Gray!20} \textbf{Our} ($T=0.6$)             & $3.72$                              & $3.07$                               & $600 $                               & $200$                                \\
	\rowcolor{Gray!40} \textbf{Our} ($T=0.4$)             & $5.44$                              & $3.06$                               & $400 $                               & $187$                                \\
	\rowcolor{Gray!60} \textbf{Our} ($T=0.2$)             & $14.38$                             & $3.06$                               & $200 $                               & $176$    
	\\
	\midrule
ARDM \citep{hoogeboom2022autoregressive}    & $-$                             & $ 2.69 $                                & $3072 $                                                                      \\
	
	VDM\cite{kingma2021variational}                           & $4.0 $                             & $ 2.49 $                             & $1000$                                                                      \\
	D3PMs \cite{austin2021structured}                         & $7.34 $                             & $ 3.43 $                                & $1000$                                                                      \\
	DDPM \cite{ho2020denoising}                               & $3.21 $                             & $3.75$                               & $1000$                                                                      \\
	\midrule
	Gotta Go Fast \cite{jolicoeur2021gotta}                   & $2.44 $                             & $ - $                                & $180 $                                                                      \\
	LSGM \citep{vahdat2021score}                                           & $2.10 $                             & $2.87$                               & $120/138 $                                                                      \\
	ARDM-P \citep{hoogeboom2022autoregressive}    & $-$                             & $ 2.68/2.74 $                                & $200/50 $                                                                      \\
	
	\bottomrule[1pt]
\end{tabular}
\\[2ex]
		\begin{minipage}{\textwidth}
			\centering
			\scriptsize
\begin{tabular}{ccc}
	
	Real data                                                                          & ScoreSDE ($T=0.4$) & \textbf{Our} ($T=0.4$) \\
	\includegraphics[width=.3\textwidth]{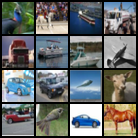}           &
	\includegraphics[width=.3\textwidth]{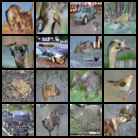} &
	\includegraphics[width=.3\textwidth]{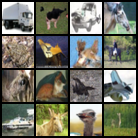}
\end{tabular}

			\captionof{figure}{Visualization of some samples on \cifar.}
			\label{fig:cifarsamples}
		\end{minipage}
	\end{minipage}\hfill
	\begin{minipage}{0.4\textwidth}
		\centering
		\begin{minipage}{\textwidth}
			\centering
			\captionof{table}{Experiment results on \mnist. For our method, (S) is for the extension in \cref{sec:densityest}}
			\label{tab:mnist}
			\scriptsize
\setlength\tabcolsep{3pt}
\setlength\extrarowheight{-4pt}
\begin{tabular}{r|clccc}
	\toprule[1pt]
	                                         & \textbf{\textsc{nfe}}($\downarrow$) & \multicolumn{2}{c}{\textbf{\textsc{bpd}} ($\downarrow$)}                \\
	\textbf{Model}                           & (\gls{ODE})\;\;\;\;                 &                                                          &              \\
	\midrule
	ScoreSDE                                 & $300$                               & $1.16$                                                   &              \\

	\midrule

	ScoreSDE ($T=0.6$)                       & $258$                               & $1.16$                                                   &              \\
	\rowcolor{Gray!20}\textbf{Our} ($T=0.6$) & $258$                               & $1.16$                                               & $1.14$ (S)   \\

	\midrule

	ScoreSDE ($T=0.4$)                       & $235$                               & $1.25$                                                   &              \\
	\rowcolor{Gray!20}\textbf{Our} ($T=0.4$) & $235$                               & $1.17$                                               & $1.16$ (S)   \\

	\midrule

	ScoreSDE  ($T=0.2$)                      & $191$                               & $1.75$                                                   &              \\
	\rowcolor{Gray!20}\textbf{Our} ($T=0.2$) & $191$                               & $1.33$                                               & $1.31$   (S) \\
	\bottomrule[1pt]
\end{tabular}

		\end{minipage}\\[1ex]
		\begin{minipage}{\textwidth}
			\includegraphics[width=\textwidth]{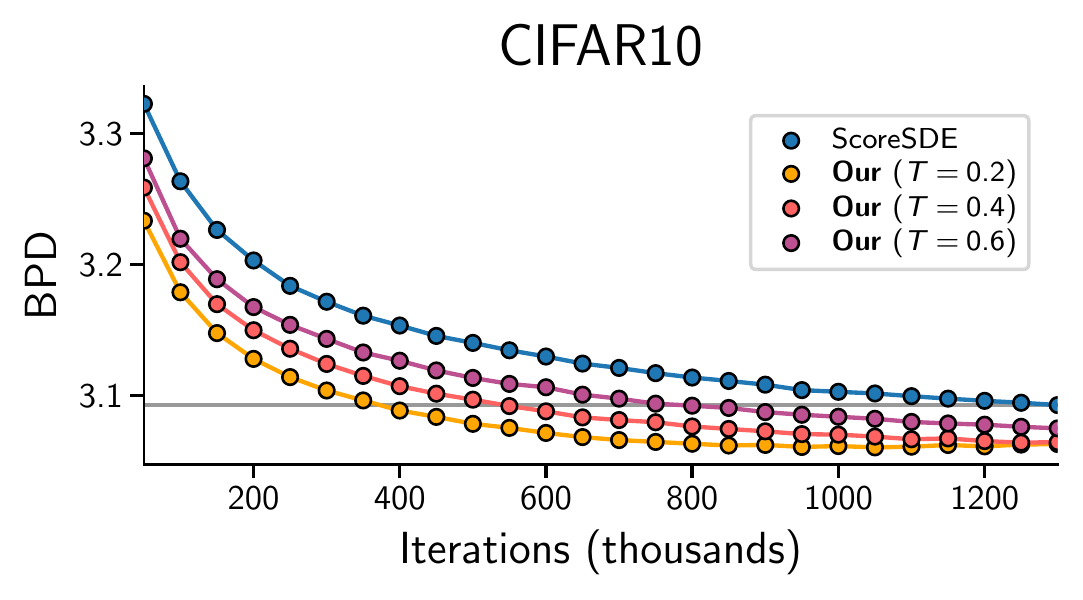}
			\captionof{figure}{Training curves of score models for different diffusion time $T$, recorded during the span of $1.3$ millions iterations.}
			\label{fig:cifar10bpd}
		\end{minipage}
	\end{minipage}
\end{figure}

\begin{minipage}{.55\textwidth}
	\paragraph{Reducing $T$ with auxiliary models.}
	We now turn attention to \cref{prop_tstarstar}, according to which it is possible to set a preferable (in terms of efficiency) $\tau\leq T^\star$, and obtain a comparable (or better) performance than the baseline model that uses $T^\star$.
	For \mnist, setting $\tau=0.4$ produces good performance both in terms of \gls{BPD} (\cref{tab:mnist}) and visual sample quality (\cref{fig:mnistsamples}).
	We also consider the sequential extension to compute the likelihood, but remark marginal improvements compared to a concurrent implementation.
	Similarly, in \cref{tab:cifar10} we observe how our method achieves better \gls{BPD} than the baseline, using $\tau=0.6$.
	Moreover, our approach outperforms the baselines in terms of \gls{FID} score (additional non-curated samples in the Appendix).
\end{minipage}\hfill
\begin{minipage}{.4\textwidth}
	\captionof{figure}{Visualization of some samples}
	\label{fig:mnistsamples}
	\scriptsize
\begin{tabular}{cc}
	\toprule[1pt]
	ScoreSDE \cite{song2020score} ($T = 1$, \gls{BPD} $=1.16$)                       \\
	\includegraphics[width=.9\textwidth]{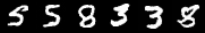}  \\
	\midrule[1pt]
	ScoreSDE \cite{song2020score} ($T = 0.4$, \gls{BPD} $=1.25$)                     \\
	\includegraphics[width=.9\textwidth]{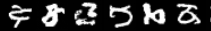} \\
	\midrule[1pt]
	\textbf{Our} ($T = 0.4$, \gls{BPD} $=1.17$)                                      \\
	\includegraphics[width=.9\textwidth]{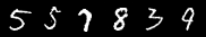}      \\
	\bottomrule[1pt]
\end{tabular}

\end{minipage}

\paragraph{Training and sampling efficiency}
In \cref{fig:cifar10bpd} the horizontal line corresponds to the best performance of a fully trained baseline model \cite{song2020score}.
To achieve the same performance of the baseline, variants of our method require fewer iterations, which translate in training efficiency.
For the sake of fairness, the total training cost of our method should account for the auxiliary model training, which however can be done concurrently to the diffusion process.
As an illustration for \cifar, using four GPUs, the baseline model requires $\sim 6.4$ days of training. With our method we trained the auxiliary and diffusion models for $\sim 2.3$ and $2$ days respectively, leading to a total training time of $\max\{2.3,2\}=2.3$ days.
Similar training curves can be obtained for the \mnist dataset, where the training time for \gls{DPGMM} is negligible.
Sampling speed benefits are evident from \cref{tab:mnist,tab:cifar10}.
When considering the \gls{SDE} version of the methods the number of sampling steps can decrease linearly with $T$, in accordance to theory \cite{Kloeden1989}, while retaining good \gls{BPD} and \gls{FID} scores.
Similarly, although not in a linear fashion, the number of steps of the \gls{ODE} samplers can be reduced by using a smaller diffusion time $T$.

\section{Conclusion}\label{sec:conclusion}

Diffusion-based generative models emerged as an extremely competitive approach for a wide range of application domains.
In practice, however, these models are resource hungry, both for their training and for sampling new data points.
In this work, we have introduced the key idea of considering diffusion times $T$ as a free variable which should be chosen appropriately.
We have shown that the choice of $T$ introduces a trade-off, for which an optimal ``sweet spot'' exists. 
In standard diffusion-based models, smaller values of $T$ are preferable for efficiency reasons, but sufficiently large $T$ are required to reduce approximation errors of the forward dynamics.
Thus, we devised a novel method that allows for an arbitrary selection of diffusion times, where even small values are allowed.
Our method closes the gap between practical and ideal diffusion dynamics, using an auxiliary model.
Our empirical validation indicated that the performance of our approach was comparable and often superior to standard diffusion models, while being efficient both in training and sampling.

\paragraph{Limitations and ethical concerns.}
In this work, the experimental protocol has been defined to corroborate our methodological contribution, and not to achieve state-of-the-art performance.
A more extensive empirical evaluation of model architectures, sampling methods, and additional datasets could benefit practitioners in selecting an appropriate configuration of our method. 
An additional limitation is the descriptive, and not prescriptive, nature of \cref{theo_tstar}: we know that $T^\star$ exists, but an explicit expression to identify the optimal diffusion is out of reach.
Finally, we inherit the same ethical concerns of all generative models, as they could be used to produce fake or misleading information to the public.

\begin{ack}
	MF gratefully acknowledges support from the AXA Research Fund and from the Agence Nationale de la
	Recherche (grant ANR-18-CE46-0002 and ANR-19-P3IA-0002).
\end{ack}

{
\small
\bibliographystyle{abbrv}

\input{arxiv.bbl}
}

\clearpage

\clearpage
\appendix

\section{Generic definitions and assumptions}

Our work builds upon the work in \cite{song2021maximum}, which should be considered as a basis for the developments hereafter. In this supplementary material we use the following shortened notation for a generic $\omega>0$:
\begin{equation}
	\Gauss{\omega}(\mbx)\defeq \mathcal{N}(\mbx;\mbzero,\omega\mbI).
\end{equation}

It is useful to notice that $\nabla\log(\Gauss{\omega}(\mbx))=-\frac{1}{\omega}\mbx$.

For an arbitrary probability density $p(\mbx)$ we define the convolution ($*$ operator) with $\Gauss{\omega}$ using notation
\begin{equation}
	p_\omega(\mbx)=p(\mbx)*\Gauss{\omega}(\mbx).
\end{equation}
Equivalently, $p_\omega(\mbx)=\exp(\frac{\omega}{2} \Delta)p(\mbx)$, and consequently $\frac{d p_\omega(\mbx)}{d\omega}=\frac{1}{2}\Delta p(\mbx)$, where $\Delta=\nabla^\top\nabla$. Notice that by considering the Dirac delta function $\delta(\mbx)$, we have the equality $\delta_\omega(\mbx)=\Gauss{\omega}(\mbx)$.

In the following derivations, we make use of the Stam–Gross logarithmic Sobolev inequality result in \citep[p.~562 Example 21.3]{villani2009optimal}:
\begin{equation}
	\KL{p(\mbx)}{\Gauss{\omega}(\mbx)}=\int p(\mbx)\log\left(\frac{p(\mbx)}{\Gauss{\omega}(\mbx)}\right)\dd\mbx\leq \frac{\omega}{2}\int \norm{\nabla\left(\log \frac{p(\mbx)}{ \Gauss{\omega}(\mbx)}\right)}^2 p(\mbx)\dd\mbx.
\end{equation}

\section{Proof of \cref{eq:I_ineq}}
We prove the following result
\begin{equation*}
	I(\mbs_{\mbtheta},T)\geq \underbrace{I(\gradient\log p,T)}_{\defeq K(T)}=\frac{1}{2}\int\limits_{t=0}^{T}g^2(t)\mathbb{E}_{\sim\eqref{eq:diffusion_sde}}\left[
		\norm{\gradient\log p(\mbx_t,t)-\gradient\log p(\mbx_t,t\g\mbx_0)}
		\right]^2\dd t.
\end{equation*}

\begin{proof}
	We prove that for generic positive $\lambda(\cdot)$, and $T_2>T_1$ the following holds:
	\begin{flalign}\label{eq:I_ineq_supp}
		& \int\limits_{t=T_1}^{T_2}\lambda(t)\mathrm{E}_{\sim\eqref{eq:diffusion_sde}}\left[||\mbs(\mbx_t,t)-\nabla\log p(\mbx_t,t|\mbx_0)||^2\right]\dd t\geq\int\limits_{t=T_1}^{T_2}\lambda(t)\mathrm{E}_{\sim\eqref{eq:diffusion_sde}}\left[||\nabla\log p(\mbx_t,t)-\nabla\log p(\mbx_t,t|\mbx_0)||^2\right]\dd t.
	\end{flalign}
	First we compute the functional derivative (w.r.t $\mbs$)
	\begin{flalign*}
		&\frac{\delta}{\delta \mbs}\int\limits_{t=T_1}^{T_2}\lambda(t)\mathrm{E}_{\sim\eqref{eq:diffusion_sde}}\left[||\mbs(\mbx_t,t)-\nabla\log p(\mbx_t,t|\mbx_0)||^2\right]\dd t=2\int\limits_{t=T_1}^{T_2}\lambda(t)\mathrm{E}_{\sim\eqref{eq:diffusion_sde}}\left[(\mbs(\mbx_t,t)-\nabla\log p(\mbx_t,t|\mbx_0))\right]\dd t=\\
		&2\int\limits_{t=T_1}^{T_2}\lambda(t)\mathrm{E}_{\sim\eqref{eq:diffusion_sde}}\left[(\mbs(\mbx_t,t)-\nabla\log p(\mbx_t,t))\right]\dd t,\end{flalign*}
	where we used
	\begin{flalign*}
		&\mathrm{E}_{\sim\eqref{eq:diffusion_sde}}\left[\nabla\log p(\mbx_t,t|\mbx_0)\right]=  \int\nabla\log p(\mbx,t|\mbx_0)p(\mbx,t|\mbx_0)\pdata(\mbx_0)\dd\mbx\dd\mbx_0 =\\&
		\int\nabla p(\mbx,t|\mbx_0)\pdata(\mbx_0)\dd\mbx\dd\mbx_0=\int\nabla p(\mbx,t)\dd\mbx=\mathrm{E}_{\sim\eqref{eq:diffusion_sde}}\left[\nabla\log p(\mbx_t,t)\right].
	\end{flalign*}
	Consequently we can obtain the optimal $\mbs$ through
	\begin{equation}
		\frac{\delta}{\delta \mbs}\int\limits_{t=T_1}^{T_2}\lambda(t)\mathrm{E}_{\sim\eqref{eq:diffusion_sde}}\left[||\mbs(\mbx_t,t)-\nabla\log p(\mbx_t,t|\mbx_0)||^2\right]\dd t=0\rightarrow \mbs(\mbx,t)=\nabla \log p(\mbx,t).
	\end{equation}
	Substitution of this result into \cref{eq:I_ineq_supp} directly proves the desired inequality.

	As a byproduct, we prove the correctness of \cref{eq:I_ineq}, since it is a particular case of \cref{eq:I_ineq_supp}, with $\lambda=g^2,T_1=0,T_2=T$. Since $K(T)$ is a minimum, the decomposition $I(\mbs_{\mbtheta},T)=K(T)+\cG(\mbs_{\mbtheta},T)$ implies
	$K(T)+\cG(\mbs_{\mbtheta},T)\geq K(T)\rightarrow \cG(\mbs_{\mbtheta},T)\geq 0$.
\end{proof}
\section{Proof of \cref{prop_elbo}}
\propelbo*

\begin{proof}

	We consider the pair of equations
	\begin{flalign}\label{eq:diffusionproof}
		&\dd\mbx_t=\left[-\mbf(\mbx_t,t')+g^2(t')\gradient\log q(\mbx_t,t)\right]\dd t +g(t')\dd\mbw(t),\quad \nonumber\\
		&\dd\mbx_t= \mbf(\mbx_t, t) \dd t+g(t)\dd\mbw(t),
	\end{flalign}
	where $t'=T-t$, $q$ is the density of the backward process and $p$ is the density of the forward process. These equations can be interpreted as a particular case of the following pair of  \gls{SDE}s (corresponding to \cite{huang2021variational} eqn (4) and (17)\footnote{Notice that our notation for the roles of $p,q$ is swapped w.r.t. \cite{huang2021variational}}).
	\begin{flalign}\label{eq:huangdiffusionproof}
		&\dd\mbx_t=\underbrace{\left[-\mbf(\mbx_t,t')+g^2(t')\gradient\log q(\mbx_t,t)\right]}_{\mbmu(\mbx_t,t)}\dd t +\underbrace{g(t')}_{\sigma(t)}\dd\mbw(t),\quad \nonumber\\
		&\dd\mbx_t= \left[\underbrace{\mbf(\mbx_t, t)-g^2(t)\gradient\log q(\mbx_t,t')}_{-\mbmu(\mbx_t,t')}+\underbrace{g(t)}_{\sigma(t')}\mba(\mbx_t,t)\right]\dd t+g(t)\dd\mbw(t),
	\end{flalign}
	where \cref{eq:diffusionproof} is recovered considering $\mba(\mbx,t)=\sigma(t')\gradient\log q(\mbx,t')=g(t)\gradient\log q(\mbx,t')$. \cref{eq:huangdiffusionproof} is associated to an \gls{ELBO} (\cite{huang2021variational}, Thm 3) that is attained with equality if and only if $\mba(\mbx,t)=\sigma(t')\gradient\log q(\mbx,t')$. Consequently, we can write the following equality associated to the backward process of \cref{eq:diffusionproof}
	\begin{flalign}\label{eq:elboproof}
		\log q(\mbx,T)=\mathbb{E}\left[-\frac{1}{2}\int\limits_{0}^T||\mba(\mbx_t,t)||^2+2\nabla^\top \mbmu(\mbx_t,t') ds+\log q(\mbx_T,0)  \quad \Big\vert\mbx_0=\mbx\right],
	\end{flalign}
	where expected value is taken w.r.t. dynamics of the associated forward process.
	
	By careful inspection of the couple of equations we notice that in the process $\mbx_t$ the drift includes the $\gradient\log q(\mbx_t,t)$ term, while in our main \eqref{eq:diffusion_sde} we have $\gradient\log p(\mbx_t,t')$. In general the two vector fields do not agree. However, if we select as starting distribution of the generating process $p(\mbx,T)$, i.e. $q(\mbx,0)=p(\mbx,T)$, then $\forall t, q(\mbx,t)=p(\mbx,t')$.

	Given initial conditions, the time evolution of the density $p$ is fully described by the Fokker-Planck equation
	\begin{equation}
		\frac{d}{\dd t}p(\mbx,t)=-\nabla^\top\left( \mbf(\mbx, t)p(\mbx,t)\right)+
		\frac{g^2(t)}{2}\Delta(p(\mbx,t)),\quad
		p(\mbx,0)=\pdata(\mbx).
	\end{equation}
	Similarly, for the density $q$,
	\begin{equation}
		\frac{d}{\dd t}q(\mbx,t)=-\nabla^\top\left( -\mbf(\mbx, t')q(\mbx,t)+g^2(t')\gradient\log q(\mbx,t)q(\mbx,t)\right)+
		\frac{g^2(t')}{2}\Delta(q(\mbx,t)),\quad
		q(\mbx,0)=p(\mbx,T).
	\end{equation}

	By Taylor expansion we have
	\begin{flalign*}
		& q(\mbx,\delta t)=q(\mbx,0)+\delta t\left(\frac{d}{\dd t}q(\mbx,t)\right)_{t=0}+\mathcal{O}(\delta t^2) =\\
		&q(\mbx,0)+\delta t\left(-\nabla^\top\left( -\mbf(\mbx, T)q(\mbx,0)+g^2(T)\gradient\log q(\mbx,0)q(\mbx,0)\right)
		+\frac{g^2(T)}{2}\Delta(q(\mbx,0))\right)+\mathcal{O}(\delta t^2)=\\
		&q(\mbx,0)+\delta t\left(\nabla^\top\left( \mbf(\mbx, T)q(\mbx,0)\right)
		-\frac{g^2(T)}{2}\Delta(q(\mbx,0))\right)+\mathcal{O}(\delta t^2),
	\end{flalign*}
	and
	\begin{flalign*}
		& p(\mbx,T-\delta t)=p(\mbx,T)-\delta t\left(\frac{d}{\dd t}p(\mbx,t)\right)_{t=T}+\mathcal{O}(\delta t^2) =\\
		&p(\mbx,T)-\delta t\left(-\nabla^\top\left( \mbf(\mbx, T)p(\mbx,T)\right)
		+\frac{g^2(T)}{2}\Delta(p(\mbx,T))\right)+\mathcal{O}(\delta t^2)=\\
		&p(\mbx,T)+\delta t\left(\nabla^\top\left( \mbf(\mbx, T)p(\mbx,T)\right)
		-\frac{g^2(T)}{2}\Delta(p(\mbx,T))\right)+\mathcal{O}(\delta t^2)
	\end{flalign*}
	Since $q(\mbx,0)=p(\mbx,T)$, we finally have $q(\mbx,\delta t)-p(\mbx,T-\delta t)=\mathcal{O}(\delta t^2)$. This holds for arbitrarily small $\delta t$. By induction, with similar reasoning, we claim that $q(\mbx,t)=p(\mbx,t')$.

	This last result allows us to rewrite \cref{eq:diffusionproof} as the pair of  \gls{SDE}s
	\begin{flalign}
		&\dd\mbx_t=\left[-\mbf(\mbx_t,t')+g^2(t')\gradient\log p(\mbx_t,t')\right]\dd t +g(t')\dd\mbw(t),\quad \nonumber\\
		&\dd\mbx_t= \mbf(\mbx_t, t) \dd t+g(t)\dd\mbw(t).
	\end{flalign}
	Moreover, since $q(\mbx,T)=p(\mbx,0)=\pdata(\mbx)$, together with the result \cref{eq:elboproof}, we have the following equality
	\begin{equation}
		\log\pdata(\mbx)=\mathbb{E}\left[-\frac{1}{2}\int\limits_{0}^T||\mba(\mbx_t,t)||^2+2\nabla^\top \mbmu(\mbx_t,t') \dd t+\log p(\mbx_T,T)  \quad \vert\mbx_0=\mbx\right].
	\end{equation}

	Consequently
	\begin{flalign*}
		&\mathbb{E}_{\mbx\sim\pdata}\left[\log\pdata(\mbx)\right]=   \mathbb{E}\left[\log p(\mbx_T,T)\right]   +\mathbb{E}\left[-\frac{1}{2}\int\limits_{0}^T||\mba(\mbx_t,t)||^2+2\nabla^\top \mbmu(\mbx_t,t') \dd t  \right]=\\
		&\mathbb{E}\left[\log p(\mbx_T,T)\right] +\mathbb{E}\left[-\frac{1}{2}\int\limits_{0}^Tg(t)^2||\gradient\log p(\mbx_t,t)||^2+2\nabla^\top\left(-\mbf(\mbx_t,t)+g^2(t)\gradient\log p(\mbx_t,t)\right) \dd t  \right]=\\
		&\mathbb{E}\left[\log p(\mbx_T,T)\right] +\mathbb{E}\left[-\frac{1}{2}\int\limits_{0}^Tg(t)^2||\gradient\log p(\mbx_t,t)||^2-2g^2(t)\gradient^\top_\mbx\log p(\mbx_t,t)\gradient\log p(\mbx_t,t|\mbx_0)  \dd t  \right]\\&+\mathbb{E}\left[-\frac{1}{2}\int\limits_{0}^T2\mbf^\top(\mbx_t,t)\gradient\log p(\mbx_t,t|\mbx_0)  \dd t  \right]=\\
		&\mathbb{E}\left[\log p(\mbx_T,T)\right] +
		\mathbb{E}\left[-\frac{1}{2}\int\limits_{0}^Tg(t)^2||\gradient\log p(\mbx_t,t)-\gradient\log p(\mbx_t,t|\mbx_0)||^2\dd t\right]+\\&\mathbb{E}\left[-\frac{1}{2}\int\limits_{0}^T-g(t)^2||\gradient\log p(\mbx_t,t|\mbx_0)||^2+2\mbf^\top(\mbx_t,t)\gradient\log p(\mbx_t,t|\mbx_0)  \dd t  \right].
	\end{flalign*}
	Remembering the definitions
	\begin{flalign*}
		&K(T)=\frac{1}{2}\int\limits_{t=0}^{T}g^2(t)\mathbb{E}_{\sim\eqref{eq:diffusion_sde}}\left[||\nabla\log p(\mbx_t,t)-\nabla\log p(\mbx_t,t|\mbx_0)||\right]^2\dd t\\
		&R(T)=\frac{1}{2}\int\limits_{t=0}^{T}\E_{\sim\eqref{eq:diffusion_sde}}\left[g^2(t)\|\gradient\log p(\mbx,t\g\mbx_0)\|\right]^2-2\mbf^\top(\mbx,t)\gradient\log p(\mbx,t\g\mbx_0)\dd t,
	\end{flalign*}
	we finally conclude the proof that
	\begin{equation}
		\mathbb{E}_{\sim\eqref{eq:diffusion_sde}}[\log p(\mbx_T,T)]-K(T)+R(T)=\mathbb{E}_{\mbx\sim \pdata}[\log \pdata(\mbx)].
	\end{equation}
\end{proof}
\section{Proof of \cref{prop:kl_vanish}}
In this Section we prove the validity of \cref{prop:kl_vanish} for the case of Variance Preserving (VP) and Variance Exploding (VE) \gls{SDE}s. Remember, as reported also in main \cref{tab:diff_types}, that the above mentioned classes correspond to  $\alpha(t)=-\frac{1}{2}\beta(t),g(t)=\sqrt{\beta(t)},\beta(t)=\beta_0+(\beta_1-\beta_0)t$ and $\alpha(t)=0,g(t)=\sqrt{\frac{d\sigma^2(t)}{\dd t}},\sigma^2(t)=  \left(\frac{\sigma_{max}}{\sigma_{min}}\right)^t$ respectively.

\propklvanish*

\subsection{The variance Preserving (VP) convergence}
We associate this class of \gls{SDE}s to the Fokker Planck operator
\begin{equation}
	\mathcal{L}^\dagger(t)=\frac{1}{2}\beta(t)\nabla^\top\left(\mbx\cdot+\nabla(\cdot)\right),
\end{equation}
and consequently $\frac{d p(\mbx,t)}{\dd t}=\mathcal{L}^\dagger(t)p(\mbx,t)$.
Simple calculations show that $\lim\limits_{T\rightarrow\infty}p(\mbx,T)=\Gauss{1}(\mbx)$.

We compute bound the time derivative of the \gls{KL} term as
\begin{flalign}\label{derivkl1}
	&\frac{d}{\dd t}\KL{p(\mbx,T)}{\Gauss{1}(\mbx)}=\int \frac{d p(\mbx,t)}{\dd t} \log(\frac{p(\mbx,t)}{\Gauss{1}(\mbx)})\dd\mbx+\int \frac{p(\mbx,t)}{p(\mbx,t)}\frac{d p(\mbx,t)}{\dd t}\dd\mbx=\nonumber\\&
	\frac{1}{2}\beta(t)\int\nabla^\top\left(-\nabla\log(\Gauss{1}(\mbx))p(\mbx,t))+\nabla p(\mbx,t))\right) \log(\frac{p(\mbx,t)}{\Gauss{1}(\mbx)})\dd\mbx=\nonumber\\
	&-\frac{1}{2}\beta(t)\int p(\mbx,t)\left(-\nabla\log(\Gauss{1}(\mbx))+\nabla\log p(\mbx,t))\right)^\top \nabla(\log(\frac{p(\mbx,t)}{\Gauss{1}(\mbx)}))\dd\mbx=\nonumber\\&
	-\frac{1}{2}\beta(t)\int p(\mbx,t)\nabla(\log(\frac{p(\mbx,t)}{\Gauss{1}(\mbx)}))^\top \nabla(\log(\frac{p(\mbx,t)}{\Gauss{1}(\mbx)}))\dd\mbx=-\frac{1}{2}\beta(t)\int p(\mbx,t)||\nabla(\log(\frac{p(\mbx,t)}{\Gauss{1}(\mbx)}))||^2\dd\mbx\nonumber\\&
	\leq -\beta(t) \KL{p(\mbx,T)}{\Gauss{1}(\mbx)}.
\end{flalign}
We then apply Gronwall's inequality \citep{villani2009optimal} to $\frac{d}{\dd t}\KL{p(\mbx,T)}{\Gauss{1}(\mbx)}\leq -\beta(t) \KL{p(\mbx,T)}{\Gauss{1}(\mbx)}$ to claim
\begin{equation}
	\KL{p(\mbx,T)}{\Gauss{1}(\mbx)}\leq \KL{p(\mbx,0)}{\Gauss{1}(\mbx)}\exp(-\int_0^T \beta(s)ds).
\end{equation}
\subsection{The Variance Exploding (VE) convergence}

The first step is to bound the derivative w.r.t to $\omega$ of the  divergence $\KL{p_\omega(\mbx)}{\Gauss{\omega}(\mbx)}$, i.e.
\begin{flalign}\label{derivkl2}
	&\frac{d}{d\omega}\KL{p_\omega(\mbx)}{\Gauss{\omega}(\mbx)}=\int \frac{d p_\omega(\mbx)}{d\omega} \log(\frac{p_\omega(\mbx)}{\Gauss{\omega}(\mbx)})\dd\mbx+\int \frac{p_\omega(\mbx)}{p_\omega(\mbx)}\frac{d p_\omega(\mbx)}{d\omega}\dd\mbx-\int\frac{p_\omega(\mbx)}{\Gauss{\omega}(\mbx)}\frac{d \Gauss{\omega}(\mbx)}{d\omega}\dd\mbx=\nonumber\\&
	\frac{1}{2}\int \left(\Delta p_\omega(\mbx)\right)\log(\frac{p_\omega(\mbx)}{\Gauss{\omega}(\mbx)})-\left(\Delta \Gauss{\omega}(\mbx)\right) \frac{p_\omega(\mbx)}{\Gauss{\omega}(\mbx)}\dd\mbx=\nonumber\\&
	\frac{1}{2}\int \nabla^\top\left(p_\omega(\mbx)\nabla\log p_\omega(\mbx)\right)\log(\frac{p_\omega(\mbx)}{\Gauss{\omega}(\mbx)})-\nabla^\top\left(\Gauss{\omega}(\mbx)\nabla\log \Gauss{\omega}(\mbx)\right) \frac{p_\omega(\mbx)}{\Gauss{\omega}(\mbx)}\dd\mbx=\nonumber\\&
	-\frac{1}{2}\int \left(p_\omega(\mbx)\nabla\log p_\omega(\mbx)\right)^\top\nabla(\log(\frac{p_\omega(\mbx)}{\Gauss{\omega}(\mbx)}))-\left(\Gauss{\omega}(\mbx)\nabla\log \Gauss{\omega}(\mbx)\right)^\top\nabla( \frac{p_\omega(\mbx)}{\Gauss{\omega}(\mbx)})\dd\mbx=\nonumber
	\\&
	-\frac{1}{2}\int \left(p_\omega(\mbx)\nabla\log p_\omega(\mbx)\right)^\top\nabla(\log(\frac{p_\omega(\mbx)}{\Gauss{\omega}(\mbx)}))-\left(p_\omega(\mbx)\nabla\log \Gauss{\omega}(\mbx)\right)^\top\nabla(\log( \frac{p_\omega(\mbx)}{\Gauss{\omega}(\mbx)}))\dd\mbx=
	\nonumber\\&-\frac{1}{2}\int p_\omega(\mbx)||\nabla(\log( \frac{p_\omega(\mbx)}{\Gauss{\omega}(\mbx)})) ||^2 \dd\mbx \leq -\frac{1}{\omega}\KL{p_\omega(\mbx)}{\Gauss{\omega}(\mbx)}.
\end{flalign}
Consequently, using again Gronwall inequality, for all $\omega_1>\omega_0>0$ we have
\begin{flalign*}
	&\KL{p_{\omega_1}(\mbx)}{\Gauss{\omega_1}(\mbx)}\leq \KL{p_{\omega_0}(\mbx)}{\Gauss{\omega_0}(\mbx)}\exp(-(\log(\omega_1)-\log(\omega_0)))=\\
	&\KL{p_{\omega_0}(\mbx)}{\Gauss{\omega_0}(\mbx)}\omega_0\frac{1}{\omega_1}.
\end{flalign*}
This can be directly applied to obtain the bound for VE \gls{SDE}. Consider $\omega_1=\sigma^2(T)-\sigma^2(0)$ and $\omega_0=\sigma^2(\tau)-\sigma^2(0)$ for an arbitrarily small $\tau<T$. Then, since for the considered class of variance exploding \gls{SDE} we have $p(\mbx,T)=p_{\sigma^2(T)-\sigma^2(0)}(\mbx)$
\begin{equation}
	\KL{p(\mbx,T)}{ \Gauss{\sigma^2(T)-\sigma^2(0)}(\mbx)}\leq C\frac{1}{\sigma^2(T)-\sigma^2(0)}
\end{equation}
where $C=\KL{p(\mbx,\tau)}{\Gauss{\sigma^2(\tau)-\sigma^2(0)}(\mbx)}$

\section{Proof of \cref{prop:score_error}}
\propscoreerror*

\begin{proof}
	For $\mbtheta_1$ defined as in the lemma, $I(\mbs_{\mbtheta_1},T_1)=K(T_1)+\cG(\mbs_{\mbtheta_1},T_1)$.
	Next, select $T_2>T_1$. Then, for a generic $\mbtheta$, including $\mbtheta_2$,
	\begin{flalign*}
		& I(\mbs_{\mbtheta},T_2) =
		\underbrace{\int\limits_{t=0}^{T_1}g^2(t)\mathbb{E}_{\sim\eqref{eq:diffusion_sde}}\left[||\mbs_{\mbtheta}(\mbx_t,t)-\gradient\log p(\mbx_t,t|\mbx_0)||^2\right]\dd t}_{= I(\mbs_{\mbtheta},T_1) \geq K(T_1)+\cG(\mbs_{\mbtheta_1},T_1) = I(\mbs_{\mbtheta_1},T_1)}+\\&\underbrace{\int\limits_{t=T_1}^{T_2}g^2(t)\mathbb{E}_{\sim\eqref{eq:diffusion_sde}}\left[||\mbs_{\mbtheta}(\mbx_t,t)-\gradient\log p(\mbx_t,t|\mbx_0)||^2\right]\dd t}_{\geq\int\limits_{t=T_1}^{T_2}g^2(t)\mathbb{E}_{\sim\eqref{eq:diffusion_sde}}\left[||\gradient\log p(\mbx_t,t)-\gradient\log p(\mbx_t,t|\mbx_0)||^2\right]\dd t= K(T_2)-K(T_1)}\geq
		\cG(\mbs_{\mbtheta_1},T_1) + K(T_2),
	\end{flalign*}
	from which $\cG(\mbs_{\mbtheta},T_2)=I(\mbs_{\mbtheta},T_2)-K(T_2)\geq\cG(\mbs_{\mbtheta_1},T_1)$.
\end{proof}

\section{Proof of \cref{theo_tstar}}
\theotstar*

\begin{proof}
	It is trivial to verify that
	since the optimal gap term $\cG(\mbs_\mbtheta, T)$ is an increasing function in $T$ \cref{prop:score_error}, then $\pdv{\cG}{T}\geq 0$.Then, we study the sign of the \gls{KL} derivative, which is always negative as shown by
	\cref{derivkl1} and \cref{derivkl2} (where we also notice $\frac{\dd}{\dd t}=\frac{\dd \omega}{\dd t}\frac{\dd}{\dd \omega}$ keep the sign). Moreover, we know that that $\lim\limits_{T\rightarrow\infty} \pdv{\textsc{kl}}{T}=0$.
	Then, the function $\pdv{\Lelbo}{T}=\pdv{\cG}{T}+\pdv{\textsc{kl}}{T}$ has at least one zero in $[0, \infty]$.
	Stricter bounding of the position of the smallest zero would require a study of the growth rates of $\cG$ and the \textsc{kl} terms for large $T$, but this is outside the scope of this paper.
\end{proof}

\section{Proof of \cref{prop_tstarstar}}

\proptstarstar*

\begin{proof}
	Since $\forall T$ we have $\Lelbo^{\mbphi}(\mbs_\mbtheta, T)\geq \Lelbo(\mbs_\mbtheta, T)$, there exists a countable set of intervals $\mathcal{I}$ contained in $[0,T^\star]$ of variable supports, where $\Lelbo^{\mbphi}$ is greater than $\Lelbo(\mbs_\mbtheta, T)$. Assuming continuity of $\Lelbo^{\mbphi}$, in these intervals is possible to find at least one $\tau\leq T^\star$ where $\Lelbo^{\mbphi^\star}(\widehat\mbs_\mbtheta, \tau) \geq \Lelbo(\widehat\mbs_\mbtheta, T^{*})$.
\end{proof}
We notice that the degenerate case $\mathcal{I}=T^\star$ is obtained only when $\forall T\leq T^\star,	\KL{p(\mbx,T)}{\nu_{\mbphi^*}(\mbx)} =
	\KL{p(\mbx,T)}{\ps(\mbx)}$. We expect this condition to never occur in practice.

\section{Experimental details}
We here give some additional details concerning the experimental (\cref{sec:experiments}) settings.

\subsection{Toy example details}
In the toy example, we use $8192$ samples from a simple Gaussian mixture with two components as target $\pdata(\mbx)$.
In detail, we have $\pdata(\mbx)=\pi\cN(1, 0.1^2) + (1-\pi)\cN(3, 0.5^2)$, with $\pi=0.3$.
The choice of Gaussian mixture allows to write down explicitly the time-varying density
\begin{align}
	p(\mbx_t, t) = \pi \cN(1, s^2(t) + 0.1^2) + (1-\pi) \cN(3, s^2(t) + 0.5^2),
\end{align}
where $s^2(t)$ is the marginal variance of the process at time $t$.
We consider a variance exploding \gls{SDE} of the type $\dd\mbx_t = \sigma^t\dd\mbw_t$, which corresponds to $s^2(t) = \frac{\sigma^{2t}-1}{2\log\sigma}$.

\begin{figure}[t]
	\centering
	\includegraphics[width=\textwidth]{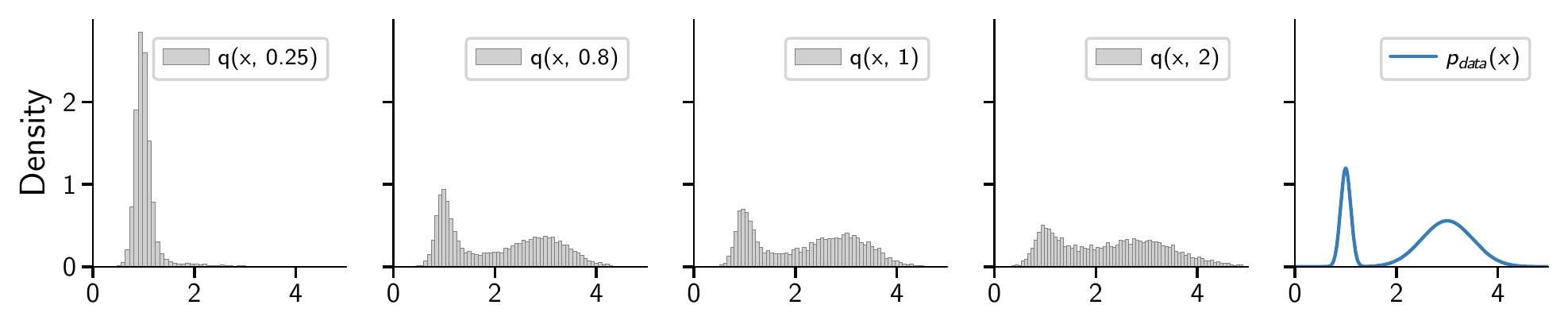}
	\caption{Visualization of few samples at different diffusion times $T$.}
\end{figure}

\subsection{\cref{sec:experiments} details}
We considered Variance Preserving \gls{SDE} with default $\beta_0,\beta_1$ parameter settings. When experimenting on \cifar we considered the \textsc{NCSN++} architecture as implemented in \cite{song2020score}. Training of the score matching network has been carried out with the default set of optimizers and schedulers of \cite{song2020score}, independently of the selected $T$.

For the \mnist dataset we reduced the architecture by considering 64 features, $\text{ch\_mult}=(1,2)$ and attention resolutions equal to 8. The optimizer has been selected as the one in the \cifar experiment but the warmup has been reduced to 1000 and the total number of iterations to 65000.
\subsection{Varying $T$}
We clarify about the $T$ truncation procedure during both training and testing. The \gls{SDE} parameters are kept unchanged irrespective of $T$. During training, as evident from \cref{eq:score_matching}, it is sufficient to sample randomly the diffusion time from distribution $\mathcal{U}(0,T)$ where $T$ can take any positive value. For testing (sampling) we simply modified the algorithmic routines to begin the reverse diffusion processes from a generic $T$ instead of the default $1.0$.

\section{Non curated samples}
We provide for completeness collection of non curated samples for the \cifar \cref{fig:cifar0.2,fig:cifar0.4,fig:cifar0.6,fig:cifar1.0} and \mnist dataset \cref{fig:mnist0.2,fig:mnist0.4,fig:mnist0.6,fig:mnist1.0}

\begin{figure}
	\centering
	\includegraphics[scale=0.5]{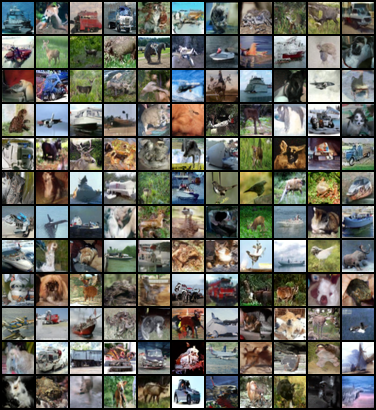}
	\includegraphics[scale=0.5]{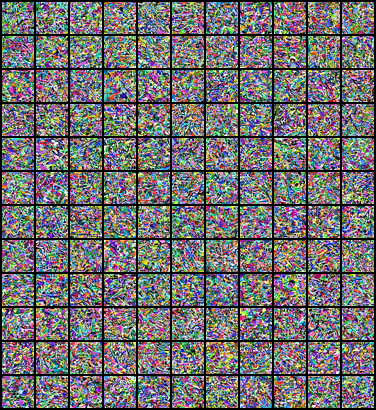}
	\caption{\cifar:Our(left) and Vanilla(right) method at $T=0.2$}
	\label{fig:cifar0.2}
\end{figure}
\begin{figure}
	\centering
	\includegraphics[scale=0.5]{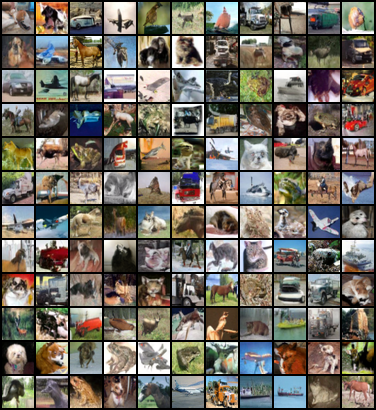}
	\includegraphics[scale=0.5]{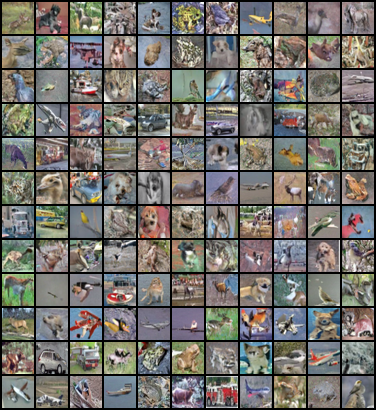}
	\caption{\cifar:Our(left) and Vanilla(right) method at $T=0.4$}
	\label{fig:cifar0.4}
\end{figure}
\begin{figure}
	\centering
	\includegraphics[scale=0.5]{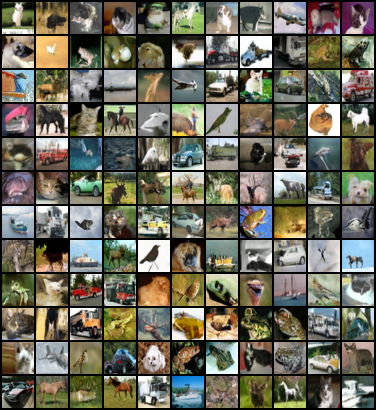}
	\includegraphics[scale=0.5]{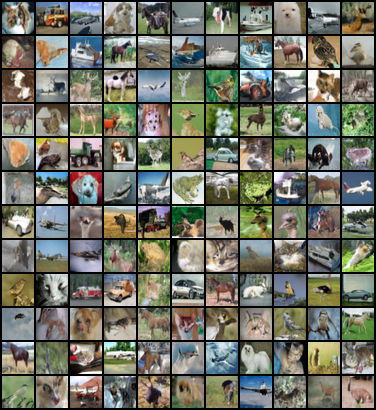}
	\caption{\cifar:Our(left) and Vanilla(right) method at $T=0.6$}
	\label{fig:cifar0.6}
\end{figure}

\begin{figure}
	\centering
	\includegraphics[scale=0.5]{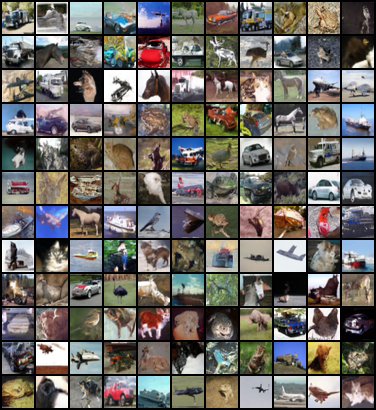}
	\caption{Vanilla method at $T=1.0$}
	\label{fig:cifar1.0}
\end{figure}

\begin{figure}
	\centering
	\includegraphics[scale=0.5]{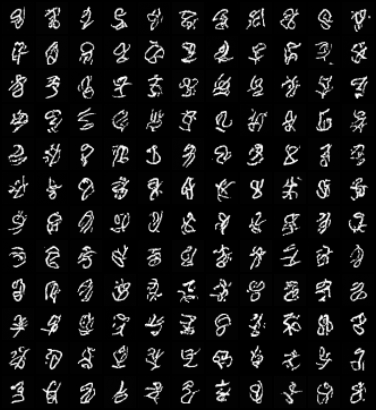}
	\includegraphics[scale=0.5]{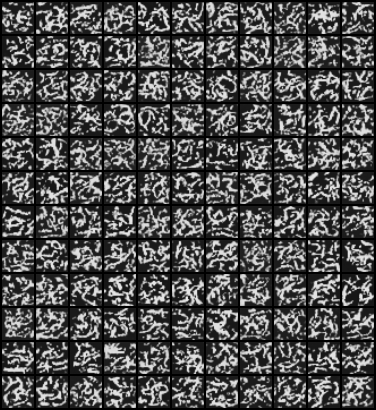}
	\caption{\mnist:Our(left) and Vanilla(right) method at $T=0.2$}
	\label{fig:mnist0.2}
\end{figure}
\begin{figure}
	\centering
	\includegraphics[scale=0.5]{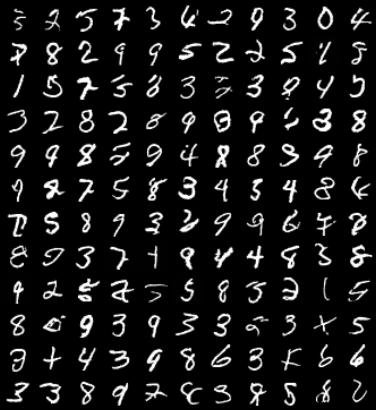}
	\includegraphics[scale=0.5]{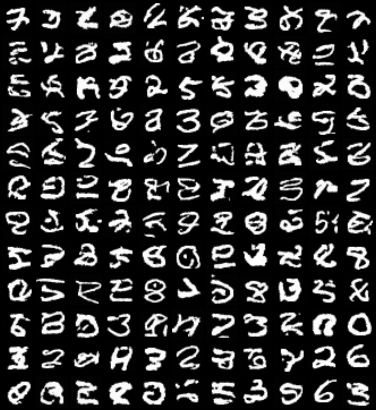}
	\caption{\mnist:Our(left) and Vanilla(right) method at $T=0.4$}
	\label{fig:mnist0.4}
\end{figure}
\begin{figure}
	\centering
	\includegraphics[scale=0.5]{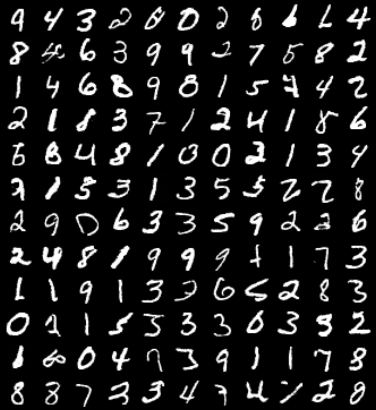}
	\includegraphics[scale=0.5]{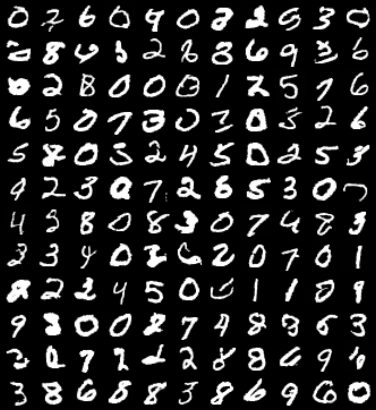}
	\caption{\mnist:Our(left) and Vanilla(right) method at $T=0.6$}
	\label{fig:mnist0.6}
\end{figure}

\begin{figure}
	\centering
	\includegraphics[scale=0.5]{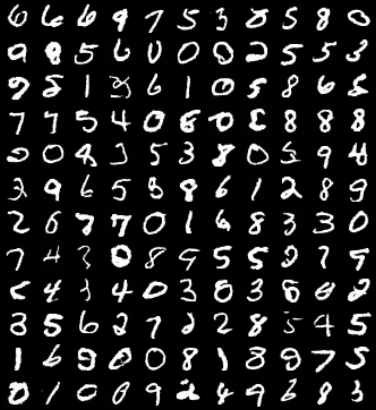}
	\caption{\mnist: Vanilla method at $T=1.0$}
	\label{fig:mnist1.0}
\end{figure}

\end{document}

%% file: bridges_3.tex
\tikzset{every picture/.style={line width=0.75pt}} 

\begin{tikzpicture}[x=0.75pt,y=0.75pt,yscale=-1,xscale=1]

\draw  [line width=2.25]  (96,110.65) .. controls (96,94.27) and (109.27,81) .. (125.65,81) .. controls (142.03,81) and (155.3,94.27) .. (155.3,110.65) .. controls (155.3,127.03) and (142.03,140.3) .. (125.65,140.3) .. controls (109.27,140.3) and (96,127.03) .. (96,110.65) -- cycle ;
\draw  [line width=2.25]  (286,110.65) .. controls (286,94.27) and (299.27,81) .. (315.65,81) .. controls (332.03,81) and (345.3,94.27) .. (345.3,110.65) .. controls (345.3,127.03) and (332.03,140.3) .. (315.65,140.3) .. controls (299.27,140.3) and (286,127.03) .. (286,110.65) -- cycle ;
\draw  [line width=2.25]  (392,261.65) .. controls (392,245.27) and (405.27,232) .. (421.65,232) .. controls (438.03,232) and (451.3,245.27) .. (451.3,261.65) .. controls (451.3,278.03) and (438.03,291.3) .. (421.65,291.3) .. controls (405.27,291.3) and (392,278.03) .. (392,261.65) -- cycle ;
\draw  [line width=2.25]  (288,213.65) .. controls (288,197.27) and (301.27,184) .. (317.65,184) .. controls (334.03,184) and (347.3,197.27) .. (347.3,213.65) .. controls (347.3,230.03) and (334.03,243.3) .. (317.65,243.3) .. controls (301.27,243.3) and (288,230.03) .. (288,213.65) -- cycle ;
\draw  [line width=2.25]  (99,181.65) .. controls (99,165.27) and (112.27,152) .. (128.65,152) .. controls (145.03,152) and (158.3,165.27) .. (158.3,181.65) .. controls (158.3,198.03) and (145.03,211.3) .. (128.65,211.3) .. controls (112.27,211.3) and (99,198.03) .. (99,181.65) -- cycle ;
\draw  [line width=2.25]  (100,245.65) .. controls (100,229.27) and (113.27,216) .. (129.65,216) .. controls (146.03,216) and (159.3,229.27) .. (159.3,245.65) .. controls (159.3,262.03) and (146.03,275.3) .. (129.65,275.3) .. controls (113.27,275.3) and (100,262.03) .. (100,245.65) -- cycle ;
\draw  [line width=2.25]  (102,310.65) .. controls (102,294.27) and (115.27,281) .. (131.65,281) .. controls (148.03,281) and (161.3,294.27) .. (161.3,310.65) .. controls (161.3,327.03) and (148.03,340.3) .. (131.65,340.3) .. controls (115.27,340.3) and (102,327.03) .. (102,310.65) -- cycle ;
\draw    (155.3,110.65) .. controls (194.7,81.1) and (245.45,72.89) .. (284.24,108.97) ;
\draw [shift={(286,110.65)}, rotate = 224.26] [fill={rgb, 255:red, 0; green, 0; blue, 0 }  ][line width=0.08]  [draw opacity=0] (8.93,-4.29) -- (0,0) -- (8.93,4.29) -- cycle    ;
\draw [color={rgb, 255:red, 126; green, 211; blue, 33 }  ,draw opacity=1 ][line width=1.5]    (158.4,113.96) .. controls (192.22,148.5) and (239.44,145.57) .. (286,110.65) ;
\draw [shift={(155.3,110.65)}, rotate = 48.18] [fill={rgb, 255:red, 126; green, 211; blue, 33 }  ,fill opacity=1 ][line width=0.08]  [draw opacity=0] (11.61,-5.58) -- (0,0) -- (11.61,5.58) -- cycle    ;
\draw [color={rgb, 255:red, 245; green, 166; blue, 35 }  ,draw opacity=1 ][line width=2.25]    (164,181.62) .. controls (237.05,180.64) and (270.87,154.03) .. (288.42,125.75) ;
\draw [shift={(158.3,181.65)}, rotate = 0.26] [fill={rgb, 255:red, 245; green, 166; blue, 35 }  ,fill opacity=1 ][line width=0.08]  [draw opacity=0] (14.29,-6.86) -- (0,0) -- (14.29,6.86) -- cycle    ;
\draw [color={rgb, 255:red, 74; green, 144; blue, 226 }  ,draw opacity=1 ][line width=1.5]    (169.43,244.01) .. controls (186.58,244.06) and (202.1,237.69) .. (217.62,230.7) .. controls (219.25,229.97) and (220.89,229.23) .. (222.53,228.48) .. controls (222.67,228.42) and (222.81,228.35) .. (222.95,228.29) .. controls (240.23,220.48) and (258.56,212.68) .. (287.97,212.15)(169.59,247) .. controls (187.12,247.06) and (202.99,240.58) .. (218.85,233.44) .. controls (220.49,232.7) and (222.12,231.96) .. (223.76,231.22) .. controls (223.9,231.15) and (224.05,231.09) .. (224.19,231.02) .. controls (241.16,223.36) and (259.15,215.67) .. (288.03,215.15) ;
\draw [shift={(159.3,245.65)}, rotate = 1.41] [fill={rgb, 255:red, 74; green, 144; blue, 226 }  ,fill opacity=1 ][line width=0.08]  [draw opacity=0] (11.61,-5.58) -- (0,0) -- (11.61,5.58) -- cycle    ;
\draw [color={rgb, 255:red, 208; green, 2; blue, 27 }  ,draw opacity=1 ][line width=2.25]    (166.49,310.79) .. controls (224.91,312.54) and (362.56,320.61) .. (399.42,283.75) ;
\draw [shift={(161.3,310.65)}, rotate = 1.41] [fill={rgb, 255:red, 208; green, 2; blue, 27 }  ,fill opacity=1 ][line width=0.08]  [draw opacity=0] (14.29,-6.86) -- (0,0) -- (14.29,6.86) -- cycle    ;
\draw [color={rgb, 255:red, 74; green, 144; blue, 226 }  ,draw opacity=1 ][line width=1.5]    (357.81,213.87) .. controls (355.36,213.44) and (358.58,213.98) .. (361.55,214.52) .. controls (384.43,218.68) and (393.11,222.91) .. (401.77,241.11)(357.3,216.82) .. controls (354.85,216.4) and (358.06,216.94) .. (361.01,217.48) .. controls (382.54,221.38) and (390.88,225.2) .. (399.06,242.39) ;
\draw [shift={(347.3,213.65)}, rotate = 9.21] [fill={rgb, 255:red, 74; green, 144; blue, 226 }  ,fill opacity=1 ][line width=0.08]  [draw opacity=0] (11.61,-5.58) -- (0,0) -- (11.61,5.58) -- cycle    ;
\draw  [dash pattern={on 4.5pt off 4.5pt}]  (421.65,144) -- (421.65,229) ;
\draw [shift={(421.65,232)}, rotate = 270] [fill={rgb, 255:red, 0; green, 0; blue, 0 }  ][line width=0.08]  [draw opacity=0] (8.93,-4.29) -- (0,0) -- (8.93,4.29) -- cycle    ;
\draw [shift={(421.65,141)}, rotate = 90] [fill={rgb, 255:red, 0; green, 0; blue, 0 }  ][line width=0.08]  [draw opacity=0] (8.93,-4.29) -- (0,0) -- (8.93,4.29) -- cycle    ;
\draw  [dash pattern={on 4.5pt off 4.5pt}]  (315.65,140.3) -- (416.3,140.3) ;
\draw  [dash pattern={on 4.5pt off 4.5pt}]  (32.15,111.65) -- (96,111.65) ;
\draw  [dash pattern={on 4.5pt off 4.5pt}]  (80.3,181.65) -- (99,181.65) ;
\draw  [dash pattern={on 4.5pt off 4.5pt}]  (57.3,245.65) -- (100,245.65) ;
\draw  [dash pattern={on 4.5pt off 4.5pt}]  (32.3,310.65) -- (102,310.65) ;
\draw  [dash pattern={on 4.5pt off 4.5pt}]  (80.3,114) -- (80.3,178.65) ;
\draw [shift={(80.3,181.65)}, rotate = 270] [fill={rgb, 255:red, 0; green, 0; blue, 0 }  ][line width=0.08]  [draw opacity=0] (8.93,-4.29) -- (0,0) -- (8.93,4.29) -- cycle    ;
\draw [shift={(80.3,111)}, rotate = 90] [fill={rgb, 255:red, 0; green, 0; blue, 0 }  ][line width=0.08]  [draw opacity=0] (8.93,-4.29) -- (0,0) -- (8.93,4.29) -- cycle    ;
\draw  [dash pattern={on 4.5pt off 4.5pt}]  (57.15,113.65) -- (57.15,242.65) ;
\draw [shift={(57.15,245.65)}, rotate = 270] [fill={rgb, 255:red, 0; green, 0; blue, 0 }  ][line width=0.08]  [draw opacity=0] (8.93,-4.29) -- (0,0) -- (8.93,4.29) -- cycle    ;
\draw [shift={(57.15,110.65)}, rotate = 90] [fill={rgb, 255:red, 0; green, 0; blue, 0 }  ][line width=0.08]  [draw opacity=0] (8.93,-4.29) -- (0,0) -- (8.93,4.29) -- cycle    ;
\draw  [dash pattern={on 4.5pt off 4.5pt}]  (32.15,114.65) -- (32.3,307.65) ;
\draw [shift={(32.3,310.65)}, rotate = 269.96] [fill={rgb, 255:red, 0; green, 0; blue, 0 }  ][line width=0.08]  [draw opacity=0] (8.93,-4.29) -- (0,0) -- (8.93,4.29) -- cycle    ;
\draw [shift={(32.15,111.65)}, rotate = 89.96] [fill={rgb, 255:red, 0; green, 0; blue, 0 }  ][line width=0.08]  [draw opacity=0] (8.93,-4.29) -- (0,0) -- (8.93,4.29) -- cycle    ;
\draw  [dash pattern={on 4.5pt off 4.5pt}]  (315.65,140.3) -- (347.57,140.3) ;
\draw  [dash pattern={on 4.5pt off 4.5pt}]  (317.65,184) -- (352.35,184) ;
\draw  [dash pattern={on 4.5pt off 4.5pt}]  (347.57,143.3) -- (347.57,181) ;
\draw [shift={(347.57,184)}, rotate = 270] [fill={rgb, 255:red, 0; green, 0; blue, 0 }  ][line width=0.08]  [draw opacity=0] (8.93,-4.29) -- (0,0) -- (8.93,4.29) -- cycle    ;
\draw [shift={(347.57,140.3)}, rotate = 90] [fill={rgb, 255:red, 0; green, 0; blue, 0 }  ][line width=0.08]  [draw opacity=0] (8.93,-4.29) -- (0,0) -- (8.93,4.29) -- cycle    ;




\draw (101,100) node [anchor=north west][inner sep=0.75pt]   [align=left] {$\displaystyle \pdata(\mathbf{x})$};
\draw (293,101) node [anchor=north west][inner sep=0.75pt]   [align=left] {$\displaystyle p(\mathbf{x} ,T)$};
\draw (402,252) node [anchor=north west][inner sep=0.75pt]   [align=left] {$\displaystyle \ps(\mathbf{x})$};
\draw (302,204) node [anchor=north west][inner sep=0.75pt]   [align=left] {$\displaystyle \nu (\mathbf{x})$};
\draw (102,170) node [anchor=north west][inner sep=0.75pt]   [align=left] {$\displaystyle q^{(1)}(\mathbf{x} ,T)$};
\draw (103,234) node [anchor=north west][inner sep=0.75pt]   [align=left] {$\displaystyle q^{(2)}(\mathbf{x} ,T)$};
\draw (105,299) node [anchor=north west][inner sep=0.75pt]   [align=left] {$\displaystyle q^{(3)}(\mathbf{x} ,T)$};
\draw (255,296) node [anchor=north west][inner sep=0.75pt]   [align=left] {$\displaystyle \mathbf{s}_{\boldsymbol{\theta }}$};
\draw (220,205) node [anchor=north west][inner sep=0.75pt]   [align=left] {$\displaystyle \mathbf{s}_{\boldsymbol{\theta }}$};
\draw (206,159) node [anchor=north west][inner sep=0.75pt]   [align=left] {$\displaystyle \mathbf{s}_{\boldsymbol{\theta }}$};
\draw (192,117) node [anchor=north west][inner sep=0.75pt]   [align=left] {$\displaystyle \gradient \ \log p$};

\draw  (66, 140) node [shape=circle, draw, inner sep=3pt, anchor=north west, fill=white] {\sffamily 1};
\draw  (42, 168) node [shape=circle, draw, inner sep=3pt, anchor=north west, fill=white] {\sffamily 2};
\draw  (17, 202) node [shape=circle, draw, inner sep=3pt, anchor=north west, fill=white] {\sffamily 3};
\draw (332, 155) node [shape=circle, draw, inner sep=3pt, anchor=north west, fill=white] {\sffamily b};
\draw (407, 180) node [shape=circle, draw, inner sep=3pt, anchor=north west, fill=white] {\sffamily a};

\end{tikzpicture}